%% file: main.tex
\documentclass[nohyperref]{article}

\usepackage[accepted]{icml2023}

\usepackage{amsmath}
\usepackage{amssymb}
\usepackage{mathtools}
\usepackage{amsthm}

\input{math_command}
\DeclareMathOperator*{\sgn}{sgn}

\allowdisplaybreaks[4]

\usepackage{tikz}
\newcommand*\circled[1]{\tikz[baseline=(char.base)]{
    \node[shape=circle,draw,inner sep=.5pt] (char) {#1};}}

\usepackage{microtype}
\usepackage{graphicx}
\usepackage{subfigure}
\usepackage{booktabs}

\usepackage{hyperref}

\usepackage{graphicx}
\usepackage[utf8]{inputenc} 
\usepackage[T1]{fontenc}    
\usepackage{url}            
\usepackage{booktabs}       
\usepackage{amsfonts}       
\usepackage{nicefrac}       
\usepackage{url}            
\usepackage{booktabs}       
\usepackage{amsfonts}       
\usepackage{nicefrac}       
\usepackage{microtype}      
\usepackage{nccmath}
\usepackage{times}
\usepackage{graphicx}
\usepackage{natbib}
\usepackage{algorithm}
\usepackage{algorithmic}
\usepackage{amsfonts}
\usepackage{amsmath}

\usepackage{color-edits}
\addauthor{jm}{purple}
\addauthor{pa}{green}
\addauthor{mm}{orange}
\addauthor{cc}{red}
\usepackage{mathtools}
\usepackage{amsthm}
\usepackage{graphicx}
\usepackage{enumerate}
\usepackage{algorithm}
\usepackage{algorithmic}
\usepackage{thm-restate}

\usepackage{url}
\usepackage{dsfont}
\usepackage[mathscr]{euscript}
\usepackage{booktabs}
\usepackage{makecell}
\usepackage{tabularx}
\usepackage{xcolor, colortbl}
\usepackage[normalem]{ulem}
\usepackage{soul}
\usepackage{prettyref}

\usepackage{MnSymbol}
\DeclareMathAlphabet\mathbb{U}{msb}{m}{n}
\usepackage{xpatch}

\DeclareMathOperator*{\argmin}{argmin}

\newtheorem{assumption}{Assumption}
\newtheorem{proposition}{Proposition}

\newtheorem{lemma}{Lemma}
\newtheorem{definition}{Definition}
\newtheorem*{remark*}{Remark}

\hypersetup{
  colorlinks   = true,
  urlcolor     = blue,
  linkcolor    = blue,
  citecolor    = blue
}

\usepackage[textsize=tiny]{todonotes}

\usepackage[toc, page, header]{appendix}
\icmltitlerunning{Submission and Formatting Instructions for ICML 2023}

\begin{document}

\twocolumn[
\icmltitle{Towards Understanding Generalization of Macro-AUC in Multi-label Learning}





\begin{icmlauthorlist}
\icmlauthor{Guoqiang Wu}{sdu}
\icmlauthor{Chongxuan Li}{ruc}
\icmlauthor{Yilong Yin}{sdu}
\end{icmlauthorlist}

\icmlaffiliation{sdu}{School of Software, Shandong University}
\icmlaffiliation{ruc}{Gaoling School of AI, Renmin University of China; Beijing Key Laboratory of Big Data Management and Analysis Methods, Beijing, China}

\icmlcorrespondingauthor{Guoqiang Wu}{guoqiangwu@sdu.edu.cn}

\icmlkeywords{Multi-label Learning, Statistical Learning Theory, Macro-AUC, Generalization}

\vskip 0.3in
]



\printAffiliationsAndNotice{}  

\begin{abstract}
Macro-AUC is the arithmetic mean of the class-wise AUCs in multi-label learning and is commonly used in practice.
However, its theoretical understanding is far lacking. Toward solving it, 
we characterize the generalization properties of various learning algorithms based on the corresponding surrogate losses w.r.t. Macro-AUC. We theoretically identify a critical factor of the dataset affecting the generalization bounds: \emph{the label-wise class imbalance}. Our results on the imbalance-aware error bounds show that the widely-used univariate loss-based algorithm is more sensitive to the label-wise class imbalance than the proposed pairwise and reweighted loss-based ones, which probably implies its worse performance.
Moreover, empirical results on various datasets corroborate our theory findings.
To establish it, technically, we propose a new (and more general) McDiarmid-type concentration inequality, which may be of independent interest.
\end{abstract}

\section{Introduction}
\label{sec:intro}










Multi-Label Learning (MLC)~\cite{mccallum1999multi} is an important learning task in machine learning where each instance might be associated with multiple labels. It has been widely applied in various areas, e.g., natural language processing~\cite{schapire2000boostexter}, computer vision~\cite{carneiro2007supervised}, and bioinformatics~\cite{elisseeff2001kernel}. Due to the complexity of MLC and the diverse demands of different scenarios, various measures~\cite{zhang2013review,wu2017unified} have been developed for a comprehensive evaluation, e.g., Hamming loss, ranking loss, and subset accuracy. Among them, Macro-AUC~\cite{zhang2013review} is a widely-used measure in practice. Informally, 
it is the arithmetic mean of the class-wise AUC measures, which is the focus of this paper. 

\begin{table*}[t]
\scriptsize
\caption{Summary of the main theoretical results. The contributions of this paper are highlighted in red.}
\label{main_theory_result}
\vskip 0.15in 
\begin{center}
\begin{small}
\begin{tabular}{lcccc}
\toprule
Algorithm & Surrogate loss & Generalization bound & Computation \\
\midrule
    {\color{red}$\mathcal{A}^{pa}$} & {\color{red}pairwise ($L_{pa}$)} & {\color{red}$\widehat{R}_S^{pa} (f) + O \left( \frac{1}{\sqrt{n}} \left( \frac{1}{K} \sum_{k=1}^K \sqrt{\frac{1}{\tau_k}} \right) \right)$} & $O(n^2 K)$ \\
    $\mathcal{A}^{u_1}$~\cite{boutell2004learning} & univariate ($L_{u_1}$) & {\color{red}$\frac{1}{\tau_S^*} \widehat{R}_S^{u_1} (f) + O \left( \frac{1}{\tau_S^* \sqrt{n}} \left( \frac{1}{K} \sum_{k=1}^K \sqrt{\frac{1}{\tau_k}} \right) \right)$} & $O(n K)$ &  \\
    {\color{red}$\mathcal{A}^{u_2}$} & {\color{red}reweighted univariate ($L_{u_2}$)} & {\color{red}$ \widehat{R}_S^{u_2} (f) + O \left( \frac{1}{\sqrt{n}} \left( \frac{1}{K} \sum_{k=1}^K \sqrt{\frac{1}{\tau_k}} \right) \right)$} & $O(n K)$  \\
\bottomrule
\end{tabular}
\end{small}
\end{center}
\end{table*}

Macro-AUC (and many other measures) in MLC are discontinuous and non-convex, which makes that optimizing them directly can lead to NP-hard problems~\cite{arora2009computational}. Thus, many surrogate losses are used in practice for computational efficiency. Empirically, many surrogate loss-based learning algorithms are commonly evaluated in terms of Macro-AUC, including the widely-used surrogate univariate loss-based algorithms~\cite{boutell2004learning,wu2020multi} that originally aim to optimize the Hamming loss. Theoretically, however, the understanding is far lacking. To take a step towards solving it, this paper attempts to formally answer the following question:
\begin{enumerate}[~]
    \item \emph{What is the learning guarantee of the widely-used surrogate univariate loss-based algorithms w.r.t. the Macro-AUC?}
\end{enumerate}
To answer the above question, we propose an analytical framework to characterize the generalization properties of learning algorithms w.r.t. the Macro-AUC. Inspired by the theory analyses, we also propose one pairwise surrogate loss and one reweighted univariate loss for Macro-AUC. Theoretically, we analyze the learning guarantees of algorithms with all three losses. We theoretically identify the \emph{label-wise class imbalance}, which is a factor of the dataset in MLC~\cite{tarekegn2021review,zhang2020towards}, plays a critical role in these generalization bounds.

Specifically, the pairwise loss-based learning algorithm $\mathcal{A}^{pa}$ has a label-wise class imbalance-aware leaning guarantee of $O \left( \frac{1}{\sqrt{n}} \left( \frac{1}{K} \sum_{k=1}^K \sqrt{\frac{1}{\tau_k}} \right) \right)$ (see Table~\ref{main_theory_result}), where $n$ is the sample size, $K$ is the label size, and $\tau_k \in [\frac{1}{n}, \frac{1}{2}]$ characterizes the $k$-th label class imbalance level. The smaller $\tau_k$, the higher the imbalance level. In contrast, the widely-used univariate loss-based algorithm $\mathcal{A}^{u_1}$ has an error bound of $O \left( \frac{1}{\tau_S^* \sqrt{n}} \left( \frac{1}{K} \sum_{k=1}^K \sqrt{\frac{1}{\tau_k}} \right) \right)$, where $\tau_S^* = \argmin_{k \in [K]} \tau_k$. Thus, we can observe that $\mathcal{A}^{u_1}$ is more sensitive to the label-wise class imbalance than $\mathcal{A}^{pa}$, which implies that $\mathcal{A}^{pa}$ would probably perform better than $\mathcal{A}^{u_1}$ practically, 
especially when $\frac{1}{\tau_S^*}$ is large, 
which often occurs in real datasets of MLC. Note that, computationally, $\mathcal{A}^{pa}$ can lead to a complexity of $O(n^2 K)$, which is worse than $\mathcal{A}^{u_1}$ (i.e., $O(nK)$), and it can be prohibitively costly when the sample size $n$ is large. Interestingly, our proposed reweighted univariate loss-based algorithm $\mathcal{A}^{u_2}$ has a generalization bound of $O \left( \frac{1}{\sqrt{n}} \left( \frac{1}{K} \sum_{k=1}^K \sqrt{\frac{1}{\tau_k}} \right) \right)$, which is nearly the same as $\mathcal{A}^{pa}$. This probably implies the performance superiority of $\mathcal{A}^{u_2}$ over $\mathcal{A}^{u_1}$, as well as the computational efficiency. Finally, empirical results corroborate our theory findings.


Technically, since optimizing Macro-AUC potentially involves learning with dependent examples, the existing generalization analytical techniques~\cite{wu2020multi,wu2021rethinking} for other measures in MLC cannot be applied, making it more challenging. Following the technique in Bipartite Ranking (BR, or equivalently AUC maximization in binary classification)~\cite{usunier2005generalization,amini2015learning}, we extend 
it to Macro-AUC maximization in MLC. Note that the technique in BR cannot be trivially applied in MLC due to the multiple labels (or tasks) property of MLC.\footnote{Note that one may use the union bound to combine the original bounds in BR to get the desired bound w.r.t. Macro-AUC in MLC, which would lead to a loosely bound involving a term $\log(\frac{K}{\delta})$.} Thus, we propose general techniques that include a new McDiarmid-type concentration inequality and a general generalization bound of learning multiple tasks with graph-dependent examples, which may be of independent interest. (See Appendix~\ref{sec:app_general_techniques} for details). Our generalization analyses on the Macro-AUC maximization in MLC can be viewed as an application of these general techniques.




\section{Preliminaries}

\textbf{Notations.} Let boldfaced lower and upper letters denote the vector (e.g., $\va$) and matrix (e.g., $\mA$), respectively. For a matrix $\mA$, $\va_i$, $\va^j$, and $a_{ij}$ denote its $i$-th row, $j$-th column and $(i,j)$-th element, respectively. For a vector $\va$, $a_i$ denotes its $i$-th element.
Let $[K]$ denote the set $\{1,\dots,K\}$. For a set, $|\cdot|$ denotes its cardinality. $[\![ \cdot ]\!]$ denotes the indicator function, i.e., it returns $1$ if the proposition holds and $0$ otherwise.

\subsection{Problem Setting}


Let $\mathbf{x} \in \mathcal{X} \subset \mathbb{R}^d$ and $\mathbf{y} \in \mathcal{Y} \subset \{ -1, +1 \}^K$ denote the input and output respectively, where $d$ is the feature dimension, and $K$ is the label size. $y_k = 1$ (or $-1$) indicates that the associated $k$-th label is relevant (or irrelevant).
Given a training set $S = \{ ( \mathbf{x}_i, \mathbf{y}_i ) \}_{i=1}^n$ of $n$ i.i.d. samples drawn from a distribution $P$ over $\mathcal{X} \times \mathcal{Y}$, the original goal of MLC is to learn a multi-label classifier $H: \mathbb{R}^d \to \{ -1, +1 \}^K$.

To solve MLC, a standard approach is first to learn a vector-based \emph{score function} (or predictor) $f = [f_1, \dots, f_K]: \mathcal{X} \to \mathbb{R}^K$ from a hypothesis space $\mathcal{F}$ and then get the classifier $H$ by a thresholding function. A typical goal in MLC is to learn the best predictor from the finite training data in terms of some ranking-based measure, which is usually called Multi-label Ranking~\cite{dembczynski2012consistent,wu2021rethinking}, and this is our focus in this paper.

\subsection{Evaluation Measure}

Many evaluation measures have been developed to evaluate the performance of different algorithms. Here we focus on the common measure Macro-AUC, which \emph{macro-averages} the AUC measure across all class labels. Given a dataset $S$ and a predictor $f \in \mathcal{F}$, Macro-AUC is defined as follows:\footnote{Note that here we do not adopt another common form w.r.t. the equality (i.e., $[\![ f_k(\vx_p) \geq f_k(\vx_q) ]\!]$), in order to avoid the trial zero hypothesis $f$. Besides, these two forms are nearly the same practically in evaluating algorithms.}
\begin{align*}
    \frac{1}{K} \sum_{k=1}^K \frac{1}{|S_k^+| |S_k^-|} \sum_{(p, q) \in S_k^+ \times S_k^-} [\![ f_k(\vx_p) > f_k(\vx_q) ]\!] ,
\end{align*}
where $S_k^+$ (or $S_k^-$) denotes the relevant (or irrelevant) instance index set for the label $k$. 

Maximizing Macro-AUC is equivalent to minimizing the following objective (i.e., one minus Macro-AUC):
\begin{align}
    \frac{1}{K} \sum_{k=1}^K \frac{1}{|S_k^+| |S_k^-|} \sum_{(p, q) \in S_k^+ \times S_k^-} L_{0/1}(\vx_p, \vx_q, f_k), 
\end{align}
where the $0/1$ loss function is defined as
\begin{align}
    L_{0/1}(\vx^+, \vx^-, f_k) = [\![ f_k(\vx^+) \leq f_k(\vx^-) ]\!] ,
\end{align}
in which $\vx^+$ (or $\vx^-$) denotes a relevant (or irrelevant) input for the label $k$.

\subsection{Risk}

Since Macro-AUC (or the $0/1$ loss function) is discontinuous and non-convex, optimizing it directly would lead to NP-hard problem~\cite{arora2009computational}. Practically, one often seeks (convex) surrogate losses for computational efficiency. Let $L_{\phi}: \mathcal{X} \times \mathcal{X} \times \mathcal{F}_k \rightarrow \sR_+$ denote a surrogate loss function where $\mathcal{F}_k = \{f_k: \mathcal{X} \rightarrow \sR \}$ and we will discuss its specific form in the next section. For a predictor $f \in \mathcal{F}$, the true ($0/1$) generalization (or expected) risk w.r.t. Macro-AUC is defined as
\begin{align*}
    R_{0/1}(f) = \frac{1}{K} \sum_{k=1}^K \eE_{\vx_p \sim P_k^+, \vx_q \sim P_k^-} \left[ L_{0/1}(\vx_p, \vx_q, f_k) \right], 
\end{align*}
where the conditional distribution $P_k^+ = P(\vx | y_k = 1)$ and $P_k^- = P(\vx | y_k = -1)$.
Besides, the surrogate empirical and generalization risks w.r.t. $L_{\phi}$ are defined as follows, respectively:
\begin{align*}
    \widehat{R}_{S}^{\phi}(f) = \frac{1}{K} \sum_{k=1}^K \frac{1}{|S_k^+| |S_k^-|} \sum_{(p, q) \in S_k^+ \times S_k^-} L_{\phi} (\vx_p, \vx_q, f_k) , 
\end{align*}
\begin{align}
\label{eq:surrogate_expected_risk}
    R_{\phi}(f) = \eE_{S} \left[ \widehat{R}_{S}^{\phi}(f) \right] .
\end{align}
Note that we do not define the surrogate generalization risk as the following common form
\begin{align}
\label{eq:surrogate_expected_risk_another}
    \frac{1}{K} \sum_{k=1}^K \eE_{\vx_p \sim P_k^+, \vx_q \sim P_k^-} L_{\phi} (\vx_p, \vx_q, f_k) .
\end{align}
This is because that Eq.\eqref{eq:surrogate_expected_risk} is more general than Eq.\eqref{eq:surrogate_expected_risk_another} where Eq.\eqref{eq:surrogate_expected_risk} can cover the surrogate loss $L_{\phi}$ depending on the training dataset $S$ while Eq.\eqref{eq:surrogate_expected_risk_another} cannot.\footnote{Note that, the aim we define Eq.\eqref{eq:surrogate_expected_risk} is for the convenience of analyses for $L_{u_1}$ and finally to get the bounds where terms only depend on the dataset. If we define the common form, we will eventually get the bounds involving the term depending on the distribution.} Besides, they are equal for certain losses independent of $S$.

\section{Methods}
In this section, we present the considered surrogate losses and their corresponding learning algorithms.
\subsection{Surrogate Losses}
\label{sec:surrogate_losses}
To optimize Macro-AUC, it is natural to use the following (surrogate) pairwise loss:
\begin{align}
\label{eq:surrogate_pa}
    L_{pa}(\vx^+, \vx^-, f_k) = \ell \left( f_k(\vx^+) - f_k(\vx^-) \right) ,
\end{align}
where the base loss function $\ell(t)$ could be many popular (margin-based) loss functions, e.g., the hinge loss $\ell(t) = \max(0, 1 - t)$, the logistic function $\ell(t) = \log_2 (1 + \exp(-t))$, and so on. A natural property of the base loss is that it is an upper bound of the original $0/1$ loss, i.e., $\ell(t) \geq [\![ t \leq 0 ]\!]$.
Note that minimizing this pairwise loss-based risk leads to a computational complexity of $O(n^2)$, which could be prohibitively costly when the sample size $n$ is large.\footnote{Note that there are possible ways to accelerate it for certain base losses.}



The widely-used univariate loss that originally aims to optimize the Hamming Loss measure~\cite{boutell2004learning,wu2020multi}, could be also viewed as a surrogate loss $L_{u_1}$ for Macro-AUC.
Its original empirical risk can be written as:
\begin{align*}
    & \widehat{R}_{S}^{u_1}(f) =  \frac{1}{K} \sum_{k=1}^K \frac{1}{n} \sum_{i=1}^n \ell (y_{ik} f_k(\vx_i)) = \frac{1}{K} \sum_{k=1}^K \frac{1}{|S_k^+| |S_k^-|} \times \\
    & \sum_{(p, q) \in S_k^+ \times S_k^-} \left (\frac{|S_k^+|}{n} \ell ( f_k(\vx_p)) + \frac{|S_k^-|}{n} \ell ( -f_k(\vx_q)) \right) .
\end{align*}
Thus, we can define this surrogate univariate loss $L_{u_1}$ w.r.t. Macro-AUC as follows:
\begin{align}
\label{eq:surrogate_u1}
    L_{u_1}(\vx^+, \vx^-, f_k) = \frac{|S_k^+|}{n} \ell \left ( f_k(\vx^+) \right ) + \frac{|S_k^-|}{n} \ell \left ( -f_k(\vx^-) \right).
\end{align}
Note that this surrogate loss cannot strictly upper bound the $0/1$ loss, i.e., $L_{0/1} \nleq L_{u_1}$. The upper bound property of the surrogate loss w.r.t. the $0/1$ loss is critical to provide its generalization analysis w.r.t. the $0/1$ loss and we discuss it in detail later. Besides, note that we cannot define the generalization risk w.r.t. $L_{u_1}$ by Eq.\eqref{eq:surrogate_expected_risk_another} due to its dependency on the dataset $S$.

    

To upper bound the $0/1$ loss (i.e., $L_{0/1}$), here we propose a new reweighted univariate surrogate loss $L_{u_2}$ with computational efficiency, which is defined as below:
\begin{align}
\label{eq:surrogate_u2}
    L_{u_2}(\vx^+, \vx^-, f_k) =  \ell \left( f_k(\vx^+) \right) +  \ell \left( -f_k(\vx^-) \right) .
\end{align}
Then, we can write its empirical risk as
\begin{align*}
    \widehat{R}_{S}^{u_2}(f) & = \frac{1}{K} \sum_{k=1}^K \frac{1}{|S_k^+| |S_k^-|} \sum_{(p, q) \in S_k^+ \times S_k^-} L_{u_2} (\vx_p, \vx_q, f_k) \\
    & = \frac{1}{K} \sum_{k=1}^K \sum_{i=1}^n \bigg( [\![ y_{ik} = 1 ]\!] \frac{1}{|S_k^+|} \ell( f_k(\vx_i)) \ + \\
    & \qquad [\![ y_{ik} \neq 1 ]\!] \frac{1}{|S_k^-|} \ell( - f_k(\vx_i)) \bigg) .
\end{align*}
We can see that, computationally, minimizing the empirical risk w.r.t. $L_{u_2}$ could lead to a complexity of $O(n)$, which is the same as $L_{u_1}$. Intuitively, this reweighted loss could be seen as a cost-sensitive loss that optimizes for balanced accuracy.



There are some relationships between these losses and we will discuss them thoroughly in the subsequent section.

\subsection{Learning Algorithm}
\label{sec:learning_algorithm}

 In the subsequent analyses, we focus on the kernel-based learning algorithms, which have been widely used in practice~\cite{elisseeff2001kernel,boutell2004learning,hariharan2010large,tan2020multi,wu2020joint} and in theory~\cite{wu2020multi,wu2021rethinking} in MLC. Note that our subsequent analyses can be extended to other forms of the hypothesis space, e.g., neural networks~\cite{anthony1999neural}. Let $\kappa: \mathcal{X} \times \mathcal{X} \rightarrow \sR$ be a Positive Definite Symmetric (PDS) kernel and denote its induced reproducing kernel Hilbert space (RKHS) as $\mathbb{H}$. Let $\Phi: \mathcal{X} \rightarrow \mathbb{H}$ be a feature mapping associated with $\kappa$. The considered kernel-based hypothesis class can be defined as
 \begin{align}
\label{eq:kernel_hypothesis}
    \mathcal{F} = \bigg\{ \mathbf{x} \mapsto \mathbf{W} ^\top \Phi(\mathbf{x}): \mathbf{W}= (\mathbf{w}_1, \ldots ,\mathbf{w}_K)^\top, \| \mathbf{w}_k \| \leq \Lambda \bigg\} ,
\end{align}
where $\| \mathbf{w}_k \|$ denotes $ \| \mathbf{w}_k \|_{\mathbb{H}}$ for convenience.

Here we consider the following three regularized learning algorithms with the aforementioned corresponding surrogate losses: 
\begin{align*}
    \mathcal{A}^{pa}: \quad & \min_{\mW} \ \widehat{R}_{S}^{pa}(f) + \lambda \| \mW \|^2 , \\
    \mathcal{A}^{u_j}: \quad & \min_{\mW} \ \widehat{R}_{S}^{u_j}(f) + \lambda \| \mW \|^2 , j = 1, 2,
\end{align*}
where $\lambda$ denotes a trade-off hyper-parameter and $\| \mathbf{W} \|$ denotes $\| \mathbf{W} \|_{\mathbb{H}, 2} = (\sum_{j=1}^c \| \mathbf{w}_j \|_{\mathbb{H}}^2)^{1/2}$ for convenience.



\section{Theoretical Results}


In this section, we mainly introduce the generalization results of the aforementioned learning algorithms with different surrogates w.r.t. the Macro-AUC measure, where the proofs of related lemmas, theorems, and corollaries are in Appendix~\ref{sec:app_macro_auc_mlc}.

Technically, to establish it, we propose new techniques including a new McDiarmid-type inequality. (Please see Section~\ref{sec:proof_sketch} for the proof sketch and Appendix~\ref{sec:app_general_techniques} for details).

Firstly, we give the following definition to characterize the label-wise class imbalance in MLC.
\begin{definition}[\textbf{Label-wise class imbalance}]
    Given a dataset $S$, define the following factor to characterize the label-wise class imbalance level for each label $k \in [K]$:
    \begin{align*}
        \tau_k = \frac{\min\{|S_k^+|, |S_k^-|\}}{n}, 
    \end{align*}
    where $\tau_k \in [\frac{1}{n}, \frac{1}{2}]$. Besides, define $\tau_S^* = \argmin_{k \in [K]} \tau_k$.
\end{definition}
From the above definition, we can see, the smaller $\tau_k$, the higher the label-wise class imbalance level. Besides, for the convenience of following discussions, we give the following definition.\footnote{Note that multi-label datasets can also be imbalanced in an inter-label way, i.e. some labels having very few positives, and other labels having many.}
\begin{definition}[\textbf{Label-wise class balanced and extremely imbalanced dataset}]
    Given a dataset $S$, we say that it is label-wise class balanced (or extremely class imbalanced) if $\forall k \in [K], \tau_k = \frac{1}{2}$ (or $\tau_k = \frac{1}{n}$) holds.\footnote{In this paper we call it balanced or extremely imbalanced for simplicity.}
\end{definition}

Then, we introduce the common mild assumptions for the subsequent analyses.
\begin{assumption}[\textbf{The common assumptions}]
\label{assump_common}
$~$
    \begin{enumerate}[(1)]
    \setlength\itemsep{-3pt}
        \item The training dataset $S = \{ ( \mathbf{x}_i, \mathbf{y}_i ) \}_{i=1}^n$ is an i.i.d. sample drawn from the distribution $P$, where $\exists \ r > 0$, it satisfies $\kappa (\mathbf{x}, \mathbf{x}) \leq r^2$ for all $\mathbf{x} \in \mathcal{X}$.
        \item The hypothesis class is defined in Eq.\eqref{eq:kernel_hypothesis}.
	    \item The base (convex) loss $\ell (z)$ is $\rho$-Lipschitz continuous and bounded by $B$.\footnote{Note that, the widely-used hinge and logistic loss are both $1$-Lipschitz continuous. Although the exponential and squared hinge losses are not globally Lipschitz continuous, they are locally Lipschitz continuous.}
    \end{enumerate}
\end{assumption}
Here we give the definition of the fractional Rademacher complexity of the loss space.
\begin{definition} [\textbf{The fractional Rademacher complexity of the loss space}]
    For each label $k \in [K]$, construct the dataset $\tS_k = \{ (\tvx_{ki}, \ty_{ki}) \}_{i=1}^{m_k} = \{ ((\tvx_{ki}^+, \tvx_{ki}^-), 1) \}_{i=1}^{m_k}$ based on the original dataset $S_k$, where $(\tvx_{ki}^+, \tvx_{ki}^-) \in S_k^+ \times S_k^-$, and let $\{ (I_{kj}, \omega_{kj})\}_{j \in [J_k]}$ be a fractional independent vertex cover of the dependence graph $G_k$ constructed over $\tS_k$ with $\sum_{j \in [J_k]} \omega_{kj} = \chi_{f} (G_k)$, where $\chi_{f} (G_k)$ is the fractional chromatic number of $G_k$. For the hypothesis space $\mathcal{F}$ and loss function $L: \mathcal{X} \times \mathcal{X} \times \mathcal{F}_k \rightarrow \sR_+$, the empirical fractional Rademacher complexity of the loss space is defined as
    \begin{align*}
        \widehat{\mathfrak{R}}_{\tS}^*(L \circ \mathcal{F}) & = \frac{1}{K} \sum_{k=1}^K \eE_{\boldsymbol{\sigma}} \Bigg[ \frac{1}{m_k} \sum_{j \in [J_k]} \omega_{kj} \times \\
        & \sup_{f \in \mathcal{F}} \left( \sum_{i \in I_{kj}} \sigma_{ki} L(\tilde{\vx}_{ki}^+, \tilde{\vx}_{ki}^-, f_k) \right) \Bigg] .
    \end{align*}
\end{definition}
Then, we give the base theorem of Macro-AUC used in the subsequent generalization analyses.
\begin{restatable}[\textbf{The base theorem of Macro-AUC}]
    {theorem}{BaseTheoremMacroAUC}
\label{thm:base_theorem_macroauc}
Assume the loss function $L_{\phi}: \mathcal{X} \times \mathcal{X} \times \mathcal{F}_k \rightarrow \sR_+$ is bounded by $M$. Then, for any $\delta > 0$, the following generalization bound holds with probability at least $1 - \delta$ over the draw of an i.i.d. sample $S$ of size $n$: 
    \begin{align*}
        \forall f \in \mathcal{F}, \ R_{\phi} (f) & \leq \widehat{R}_{S}^{\phi} (f) + 2 \widehat{\mathfrak{R}}_{\tS}^*(L_{\phi} \circ \mathcal{F}) \ + \\
        & 3 M \sqrt{ \frac{1}{2n} \log \left(\frac{2}{\delta} \right)} \left( \sqrt{\frac{1}{K} \sum_{k=1}^K \frac{1}{\tau_k}}  \right) \ .
    \end{align*}
\end{restatable}


    

Then, we analyze the relationship between the surrogate and true losses as follows.
\begin{restatable}[\textbf{The relationship between the surrogate and true losses}]
    {lemma}{RelationshipLosses}
\label{lem:relationship_losses}
Assume the base loss function upper bounds the original $0/1$ loss, i.e., $\ell(t) \geq [\![ t \leq 0 ]\!]$. Then, for any $f_k \in \mathcal{F}_k$ and $(\vx^+, \vx^-) \in S_k^+ \times S_k^-$, the following inequalities hold:
    \begin{align*}
        & L_{0/1}(\vx^+, \vx^-, f_k) \leq L_{pa}(\vx^+, \vx^-, f_k) , \\
        & L_{0/1}(\vx^+, \vx^-, f_k) \leq L_{u_2}(\vx^+, \vx^-, f_k) \leq \frac{1}{\tau_k} L_{u_1}(\vx^+, \vx^-, f_k) \\
        & \leq \frac{1 - \tau_k}{\tau_k} L_{u_2}(\vx^+, \vx^-, f_k) .
    \end{align*}
\end{restatable}
\begin{remark*}
    From this lemma, we can observe that when minimizing $L_{u_1}$, it also minimizes an upper bound of $L_{0/1}$ depending on $\frac{1}{\tau_k}$. Besides, for the second inequality involving $L_{u_1}$ and $L_{u_2}$, the bound is tight since the equality holds when $\tau_k = \frac{1}{2}$.
\end{remark*}
Based on Lemma~\ref{lem:relationship_losses}, we can get the relationship between the surrogate and true risks as follows, which is critical for the generalization analyses.
\begin{restatable}[\textbf{The relationship between the surrogate and true risks}]
    {lemma}{RelationshipRisks}
\label{lem:relationship_risks}
Assume the base loss function upper bounds the original $0/1$ loss, i.e., $\ell(t) \geq [\![ t \leq 0 ]\!]$. Then, for any $f \in \mathcal{F}$ and any sample $S \overset{i.i.d.}{\thicksim} P$, the following inequalities hold:
    \begin{align*}
        & R_{0/1}(f) \leq R_{pa}(f) , \\
        & R_{0/1}(f) \leq R_{u_2} (f) = \eE_{S} \left[ \widehat{R}_{S}^{u_2}(f) \right] \leq \eE_{S} \left[ \frac{1}{\tau_S^*}\widehat{R}_{S}^{u_1}(f) \right] \\
        & \leq \eE_{S} \left[ \frac{1 - \tau_S^*}{\tau_S^*}\widehat{R}_{S}^{u_2}(f) \right] .
    \end{align*}
\end{restatable}
\begin{remark*}
    For the second inequality involving the generalization risk w.r.t. $L_{u_1}$ and $L_{u_2}$, the bound is tight since the equality holds when $\tau_k = \frac{1}{2}$.
\end{remark*}

Next, for clear discussions, we introduce the generalization results of algorithms w.r.t. the label-wise class imbalance: general, balanced, and extremely imbalanced cases.
\subsection{General Case}
\label{sec:learning_guarantee_general_case}
Here we introduce the generalization results of algorithms w.r.t. general datasets which cover the subsequent balanced and extremely imbalanced datasets. 
\begin{restatable}[\textbf{Learning guarantee of $\mathcal{A}^{pa}$ in general case}]
    {theorem}{LearningGuaranteePa}
\label{thm:learning_guarantee_pa}
    Assume the loss $L_{\phi} = L_{pa}$, where $L_{pa}$ is defined in Eq.\eqref{eq:surrogate_pa}. Besides, {\rm Assumption~\ref{assump_common}} holds. Then, for any $\delta > 0$, with probability at least $1 - \delta$ over the draw of an i.i.d. sample $S$ of size $n$, the following generalization bound holds for any $f \in \mathcal{F}$: 
    \begin{align}
        R_{0/1} (f) \leq R_{pa} (f) \leq & \widehat{R}_{pa} (f) + \frac{4 \rho r \Lambda}{\sqrt{n}} \left (\frac{1}{K} \sum_{k=1}^K \sqrt{\frac{1}{\tau_k}} \right ) + \nonumber\\
        & 3 B \sqrt{\frac{\log(\frac{2}{\delta})}{2n}} \left ( \sqrt{\frac{1}{K} \sum_{k=1}^K \frac{1}{\tau_k}} \right ) .
    \end{align}
\end{restatable}
From this theorem, we can observe that $\mathcal{A}^{pa}$ has a label-wise class imbalance-aware learning guarantee of $O \left( \frac{1}{\sqrt{n}} \left( \frac{1}{K} \sum_{k=1}^K \sqrt{\frac{1}{\tau_k}} \right) \right) \approx O \left( \frac{1}{\sqrt{n}} \left ( \sqrt{\frac{1}{K} \sum_{k=1}^K \frac{1}{\tau_k}} \right ) \right)$ w.r.t. Macro-AUC.  

\begin{restatable}[\textbf{Learning guarantee of $\mathcal{A}^{u_1}$ in general case}]
    {theorem}{LearningGuaranteeUone}
\label{thm:learning_guarantee_u1}
    Assume the loss $L_{\phi} = \frac{1}{\tau_S^*} L_{u_1}$, where $L_{u_1}$ is defined in Eq.\eqref{eq:surrogate_u1}. Besides, {\rm Assumption~\ref{assump_common}} holds. Then, for any $\delta > 0$, with probability at least $1 - \delta$ over the draw of an i.i.d. sample $S$ of size $n$, the following generalization bound holds for any $f \in \mathcal{F}$: 
    \begin{align}
        R_{0/1} (f) \leq & \frac{1}{\tau_S^*} \widehat{R}_{u_1} (f) + \frac{4 \rho r \Lambda}{\tau_S^* \sqrt{n}} \left (\frac{1}{K} \sum_{k=1}^K \sqrt{\frac{1}{\tau_k}} \right )  + \nonumber \\
        & \frac{3 B}{\tau_S^*} \sqrt{\frac{\log(\frac{2}{\delta})}{2n}} \left ( \sqrt{\frac{1}{K} \sum_{k=1}^K \frac{1}{\tau_k}} \right ) .
    \end{align}
\end{restatable}
From this theorem, we can see that $\mathcal{A}^{u_1}$ has an imbalance-aware learning guarantee of $O \left( \frac{1}{\tau_S^* \sqrt{n}} \left( \frac{1}{K} \sum_{k=1}^K \sqrt{\frac{1}{\tau_k}} \right) \right)$ w.r.t. Macro-AUC.

\begin{restatable}[\textbf{Learning guarantee of $\mathcal{A}^{u_2}$ in general case}]
    {theorem}{LearningGuaranteeUtwo}
\label{thm:learning_guarantee_u2}
    Assume the loss $L_{\phi} = L_{u_2}$, where $L_{u_2}$ is defined in Eq.\eqref{eq:surrogate_u2}. Besides, {\rm Assumption~\ref{assump_common}} holds. Then, for any $\delta > 0$, with probability at least $1 - \delta$ over the draw of an i.i.d. sample $S$ of size $n$, the following generalization bound holds for any $f \in \mathcal{F}$: 
    \begin{align}
        R_{0/1} (f) \leq R_{u_2} (f) \leq & \widehat{R}_{u_2} (f) + \frac{8 \rho r \Lambda}{\sqrt{n}} \left (\frac{1}{K} \sum_{k=1}^K \sqrt{\frac{1}{\tau_k}} \right ) + \nonumber\\
        & 6 B \sqrt{\frac{\log(\frac{2}{\delta})}{2n}} \left ( \sqrt{\frac{1}{K} \sum_{k=1}^K \frac{1}{\tau_k}} \right ) .
    \end{align}
\end{restatable}
From this theorem, we can see that $\mathcal{A}^{u_2}$ has an imbalance-aware learning guarantee of $O \left( \frac{1}{\sqrt{n}} \left( \frac{1}{K} \sum_{k=1}^K \sqrt{\frac{1}{\tau_k}} \right) \right)$ w.r.t. Macro-AUC, which is nearly the same as $\mathcal{A}^{pa}$.

\subsection{Balanced Case}

Here we consider the balanced dataset. Note that in this case, algorithms $\mathcal{A}^{u_1}$ and $\mathcal{A}^{u_2}$ are exactly the same, which should share the same learning guarantee and it is confirmed by the following corollary.
\begin{restatable}[\textbf{Learning guarantee of $\mathcal{A}^{u_1}$ and $\mathcal{A}^{u_2}$ in balanced case}]
    {corollary}{LearningGuaranteeUBalanced}
\label{cor:learning_guarantee_u_balanced}
    Assume the loss $L_{\phi} = 2 L_{u_1} = L_{u_2}$, where $L_{u_1}$ and $L_{u_2}$ are defined in Eq.\eqref{eq:surrogate_u1} and Eq.\eqref{eq:surrogate_u2}, respectively. Besides, {\rm Assumption~\ref{assump_common}} holds and suppose $S$ is balanced. Then, for any $\delta > 0$, with probability at least $1 - \delta$ over the draw of an i.i.d. sample $S$ of size $n$, the following generalization bound holds for any $f \in \mathcal{F}$: 
    \begin{align}
        R_{0/1} (f) \leq R_{u_2} (f) = 2 R_{u_1} (f) \leq & \widehat{R}_{u_2} (f) + \frac{8 \sqrt{2} \rho r \Lambda}{\sqrt{n}} \ + \nonumber\\
        & 6 \sqrt{2} B \sqrt{\frac{\log(\frac{2}{\delta})}{2n}} , 
    \end{align}
    where $\widehat{R}_{u_2} (f) = 2 \widehat{R}_{u_1} (f)$.
\end{restatable}
Note that in this case, the same error bound of $\mathcal{A}^{u_1}$ and $\mathcal{A}^{u_2}$ confirms the validity of our analyses. From this corollary, we can see $\mathcal{A}^{u_1}$ (or $\mathcal{A}^{u_2}$) has an error bound of $O(\sqrt{\frac{1}{n}})$, which is the same as $\mathcal{A}^{pa}$ (see~Corollary~\ref{cor:learning_guarantee_pa_balanced} in Appendix~\ref{sec:app_learning_guarantee_pa_balanced}).

\subsection{Extremely Imbalanced Case}

Here we consider the extremely imbalanced datasets. In this case, the generalization results are as follows.
\begin{restatable}[\textbf{Learning guarantee of $\mathcal{A}^{pa}$ in extremely imbalanced case}]
    {corollary}{LearningGuaranteePaExtremeImbalanced}
\label{cor:learning_guarantee_pa_extreme_imbalanced}
    Assume the loss $L_{\phi} = L_{pa}$, where $L_{pa}$ is defined in Eq.\eqref{eq:surrogate_pa}. Besides, {\rm Assumption~\ref{assump_common}} holds and suppose $S$ is extremely imbalanced. Then, for any $\delta > 0$, with probability at least $1 - \delta$ over the draw of an i.i.d. sample $S$ of size $n$, the following generalization bound holds for any $f \in \mathcal{F}$: 
    \begin{align*}
        R_{0/1} (f) \leq R_{pa} (f) \leq \widehat{R}_{pa} (f) + 4 \rho r \Lambda + 3 B \sqrt{\log(\frac{2}{\delta})} .
    \end{align*}
\end{restatable}

\begin{restatable}[\textbf{Learning guarantee of $\mathcal{A}^{u_1}$ in extremely imbalanced case}]
    {corollary}{LearningGuaranteeUOneExtremeImbalanced}
\label{cor:learning_guarantee_u1_extreme_imbalanced}
    Assume the loss $L_{\phi} = n L_{u_1}$, where $L_{u_1}$ is defined in Eq.\eqref{eq:surrogate_u1}. Besides, {\rm Assumption~\ref{assump_common}} holds and suppose $S$ is extremely imbalanced. Then, for any $\delta > 0$, with probability at least $1 - \delta$ over the draw of an i.i.d. sample $S$ of size $n$, the following generalization bound holds for any $f \in \mathcal{F}$: 
    \begin{align*}
        R_{0/1} (f) \leq & n \widehat{R}_{u_1} (f) + 4 n \rho r \Lambda + 3 B n \sqrt{\frac{\log(\frac{2}{\delta})}{2}}.
    \end{align*}
\end{restatable}
From the above corollaries, we can see that $\mathcal{A}^{pa}$ has an error bound of $O(1)$ w.r.t. $n$, while $\mathcal{A}^{u_1}$ depends on $O(n)$. Besides, $\mathcal{A}^{u_2}$ has a similar error bound to $\mathcal{A}^{pa}$ (see Corollary~\ref{cor:learning_guarantee_u2_extreme_imbalanced} in Appendix~\ref{sec:app_learning_guarantee_u2_extreme_imbalanced}). One may notice that these bounds all diverge when $n \rightarrow \infty$. This may be due to the following two reasons. On one hand, learning in the extremely imbalanced case is indeed difficult. On the other hand, our analysis techniques might not be optimal w.r.t. $n$, and advanced techniques (e.g., local Rademacher-type complexity~\cite{bartlett2005local}) might be used to improve it. However, this is not our focus in this paper and we mainly focus on the generalization effect of label-wise class imbalance factors under the same framework, and the orders of different algorithms can still provide valuable insights.

\subsection{Comparison and Discussion}
 




For generalization analyses, a tighter upper bound usually implies probably better performance~\cite{mohri2018foundations}.\footnote{Note that when comparing bounds, it is usually more reasonable to compare the order of dependent variables rather than the absolute values.} In this paper, all algorithms are analyzed under the same framework and inequalities between the surrogate and true risks (or losses) are tight. Therefore, it is relatively safe to evaluate the performance of the algorithms theoretically by comparing their upper bounds. We now compare these algorithms as follows.
\begin{itemize}
\setlength\itemsep{-2pt}
\vspace{-.2cm}
    \item \textbf{$\mathcal{A}^{pa}$ vs $\mathcal{A}^{u_1}$}.  $\mathcal{A}^{pa}$ usually has a tighter bound than $\mathcal{A}^{u_1}$. Specifically, given the same hypothesis space, it is usually easier to train $\widehat{R}^{pa}_S$ than other risks, making $\widehat{R}^{pa}_S$ smaller than $\frac{1}{\tau_S^*}\widehat{R}^{u_1}_S$.\footnote{Although we can not formally express this claim, we empirically observed it in experiments.} 
    Besides, for the model complexity terms (i.e., the last two terms), $\mathcal{A}^{pa}$ has an error bound of $O \left( \frac{1}{K} \sum_{k=1}^K \sqrt{\frac{1}{\tau_k}} \right)$ while $\mathcal{A}^{u_1}$ depends on $O \left( \frac{1}{\tau_S^*} \left( \frac{1}{K} \sum_{k=1}^K \sqrt{\frac{1}{\tau_k}} \right) \right)$. 
    \item \textbf{$\mathcal{A}^{u_2}$ vs $\mathcal{A}^{u_1}$}. 
    Similarly, we argue that $\mathcal{A}^{u_2}$ usually has a tighter bound than $\mathcal{A}^{u_1}$.
    For the first risk term, $\frac{1}{\tau_S^*}\widehat{R}^{u_1}_S$ is usually comparable or even larger than $\widehat{R}^{u_2}_S$.\footnote{In some cases, the first risk term may be bigger than $1$ but we can still take insights from the error bound through the dependent variables of the model complexity.}
    For the model complexity term, $\mathcal{A}^{u_2}$ has an error bound of $O \left( \frac{1}{K} \sum_{k=1}^K \sqrt{\frac{1}{\tau_k}} \right)$ while $\mathcal{A}^{u_1}$ depends on $O \left( \frac{1}{\tau_S^*} \left( \frac{1}{K} \sum_{k=1}^K \sqrt{\frac{1}{\tau_k}} \right) \right)$.
    \item \textbf{$\mathcal{A}^{pa}$ vs $\mathcal{A}^{u_2}$}.
    $\mathcal{A}^{pa}$ and $\mathcal{A}^{u_2}$ have similar or comparable learning guarantees. Specifically, for the first risk term, $\widehat{R}^{pa}_S$ is usually comparable to $\widehat{R}^{u_2}_S$.\footnote{Note that although we cannot formally express this, the experiments verify it.} For the model term, $\mathcal{A}^{pa}$ and $\mathcal{A}^{u_2}$ are nearly the same (with only different constant) w.r.t. the label-wise class imbalance. 
    \vspace{-.2cm}
\end{itemize}
Overall, the tighter bound level of $\mathcal{A}^{pa}$ (and $\mathcal{A}^{u_2}$) over $\mathcal{A}^{u_1}$ heavily depend on $\frac{1}{\tau_S^*}$. Thus, when $\frac{1}{\tau_S^*}$ is large, $\mathcal{A}^{pa}$ and $\mathcal{A}^{u_2}$ would probably perform significantly better than $\mathcal{A}^{u_1}$. In contrast, when $\frac{1}{\tau_S^*}$ is small, $\mathcal{A}^{pa}$ and $\mathcal{A}^{u_2}$ would probably perform slightly better than or nearly comparably to $\mathcal{A}^{u_1}$. Besides, $\mathcal{A}^{pa}$ and $\mathcal{A}^{u_2}$ have nearly the same learning guarantee, thus they would probably perform comparably. Experimental results on benchmark datasets in Table~\ref{tab:benchmark_results} confirm our analyses.

Note that there may be another way to analyze the learning guarantees of $\mathcal{A}^{u_1}$ and $\mathcal{A}^{u_2}$ w.r.t. the Macro-AUC under the analytical framework of prior work~\cite{wu2020multi,wu2021rethinking} and here we focus on analyzing three algorithms under the same framework, leaving it as future work.

\textbf{Implications of the theory in real-world applications.} While the MLC datasets are usually highly (label-wise) class imbalanced in real-world applications, our theory can have valuable implications in practice. Specifically, our theoretical results on the imbalance-aware bounds show that the imbalance-aware loss-based algorithm $\mathcal{A}^{u_2}$ has a better learning guarantee w.r.t. the label-wise class imbalance than the algorithm $\mathcal{A}^{u_1}$ with the original univariate loss (e.g., cross-entropy loss), which probably implies its performance superiority. This can provide valuable insights to explain why the existing imbalance-aware reweighting losses~\cite{ridnik2021asymmetric, wu2020distribution} can have promising performance w.r.t. ranking-based measures (e.g., mean average precision (mAP)) similar to Macro-AUC in practice. Further, how to design more effective imbalance-aware loss w.r.t. specific measures and how to make these bounds tighter would inspire more effective algorithms.


\subsection{Proof Sketch}
\label{sec:proof_sketch}

Here we mainly summarize the proof sketch of the generalization results of the general case in Section~\ref{sec:learning_guarantee_general_case} as follows.

\textbf{Proof sketch:} Overall, the proof can be mainly divided into the following two steps.

\emph{Step 1:} Construct general techniques in need (see Appendix~\ref{sec:app_general_techniques} for details). First, we propose a new (and more general) McDiarmid-type inequality (i.e., Theorem~\ref{thm:new_mcdiarmid}). Then, based on it, we propose a general generalization bound (i.e., Theorem~\ref{thm:general_bound_multiple_tasks_dependent}) for the problem setting of learning multiple tasks with graph-dependent examples, which involves the fractional Rademacher complexity of the loss space.

\emph{Step 2:} Get the results of the Macro-AUC maximization (MaAUCM) in MLC by applying the generalization bound obtained in Step~1 (see Appendix~\ref{sec:app_macro_auc_mlc} for details). Firstly we transform the MaAUCM problem into the problem setting of learning multiple tasks with graph-dependent examples in Step 1 and then we get the base error bound of Macro-AUC (i.e., Theorem~\ref{thm:base_theorem_macroauc}). Next, we analyze the relationships between surrogate and true risks. Then, for each algorithm, define its specific fractional Rademacher complexity of the hypothesis space, upper bound the kernel-based one (e.g., Lemma~\ref{lem:app_frac_rade_kernel_hypothesis_pairwise}), and get the specific contraction inequality (e.g., Lemma~\ref{lem:app_contraction_pairwise}) to connect the complexity of the loss space with the complexity of the hypothesis space. Finally, we can get the desired results based on the above intermediate results.



\subsection{Consistency of Surrogate Losses}

Except for the finite-sample generalization guarantee, the consistency of surrogates is also important. Following~\cite{gao2015consistency, kotlowski2011bipartite}, here we consider the consistency of $L_{pa}$, $L_{u_1}$ and $L_{u_2}$ w.r.t. the Macro-AUC with $L_{0/1}(\mathbf{x}^+, \mathbf{x}^-, f_k) = [\![ f_k(\mathbf{x}^+) < f_k(\mathbf{x}^-) ]\!] + \frac{1}{2}[\![ f_k(\mathbf{x}^+) = f_k(\mathbf{x}^-) ]\!]$. 
The Macro-AUC maximization task of MLC can be decomposed into $K$ AUC maximization tasks of binary classification. Thus, we can investigate the consistency of these surrogates based on the previous well-studied consistency results of AUC in binary classification (or, equivalently bipartite ranking)~\cite{gao2015consistency, kotlowski2011bipartite}. Specifically, we can get the following results about these surrogates:
\begin{itemize}
\setlength\itemsep{-2pt}
\vspace{-.2cm}
    \item $L_{pa}$: Based on the previous result in binary classification with AUC maximization~\cite{gao2015consistency} (i.e., Corollary 1 on Page 4), we can get that $L_{pa}$ is consistent w.r.t. Macro-AUC with the (base) logistic loss and exponential loss. 
    Besides, based on the previous result~\cite{gao2015consistency} (i.e., Lemma 3 on Page 3), we can get that $L_{pa}$ is inconsistent w.r.t. Macro-AUC with the (base) hinge loss and absolute loss. 
    \item $L_{u_1}$: Based on the previous result~\cite{gao2015consistency} (i.e., Theorem 7 on Page 6), we can get that $L_{u_1}$ is consistent with the (base) exponential loss.
    \item $L_{u_2}$: As a reweighting univariate loss, $L_{u_2}$ involves reweighting factors depending on the dataset (i.e., $\frac{1}{|S_k^+|}$, $\frac{1}{|S_k^-|}$ for the positive and negative instances, respectively). Thus, in the infinite-sample (i.e., population) setting, it could be regarded to be dependent on the distribution, where the reweighting factors of positive and negative instances are propositional to $\frac{1}{P(y_k = 1)}$ and $\frac{1}{1 - P(y_k = 1)}$, respectively. In this case, based on the previous result of bipartite ranking~\cite{kotlowski2011bipartite} (i.e., Theorem 4.1 on Page 4), we can get that $L_{u_2}$ is consistent w.r.t. Macro-AUC with the (base) logistic loss and exponential loss.
\end{itemize}
As for the surrogates $L_{u_1}$ and $L_{u_2}$ with other base losses, we left it as future work.

\section{Related Work}




\textbf{Consistency.} \citet{gao2013consistency} studied the consistency of surrogate losses w.r.t. the Hamming and (partial) ranking measures in general. Besides, \citet{dembczynski2012consistent} presented an explicit regret bound for a surrogate univariate loss under the partial ranking measure. Notably, for the F-measure in binary classification, \citet{ye2012optimizing} justified and connected the empirical utility maximization (EUM) framework and the decision-theoretic approach (DTA), with applications to optimizing the macro-F measure in MLC.\footnote{Note that our generalization analyses of learning algorithms w.r.t. Macro-AUC is in the EUM framework.} For these two approach frameworks, \citet{dembczynski2017consistency} revisited the consistency analysis for binary classification with complex metrics, where they chose the more descriptive names Population Utility (PU) and Expected Test Utility (ETU). Further, for the F-measure in MLC, \citet{waegeman2014bayes,zhang2020convex} studied the consistency in the perspective of DTA via estimating the conditional distribution $P(\mathbf{y}|\mathbf{x})$ differently. Further, \citet{koyejo2015consistent} studied  consistent MLC approaches w.r.t. various measures in the EUM framework and 
\citet{menon2019multilabel} investigated the multi-label consistency of various reduction methods w.r.t. precision@$k$ and recall@$k$ measures.

\textbf{Generalization.} \citet{wu2020multi} studied the generalization of learning algorithms with surrogates aiming to optimize Hamming loss and subset accuracy w.r.t. these two measures, and found that the label size played an important role in the generalization bounds, which explains the empirical phenomena that when the label size is not large, optimizing Hamming loss with its surrogate can have promising performance w.r.t. subset accuracy. Further, \citet{wu2021rethinking} revisited the consistency and generalization of many surrogate loss-based algorithms w.r.t. the ranking loss measure and identified the \emph{instance-wise class imbalance} of the dataset (or distribution) plays a critical role in the generalization bounds, which could explain the empirical phenomena better than consistency.

We mention that~\citet{wu2017unified} also proposed a pairwise loss (similar to Eq.~\eqref{eq:surrogate_pa}), which omits the reweighting factor $\frac{1}{|S_k^+| |S_k^-|}$ and lacks formal generalization analyses. Besides, please see Appendix~\ref{sec:app_addtional_related_work} for detailed discussions about comparisons between a recent McDiarmid-type concentration inequality~\cite{zhang2019mcdiarmid} and ours for data with graph dependence.

\section{Experiments}


\begin{table*}[ht]
\scriptsize
\caption{Basic statistics of the benchmark datasets. Denote the label-wise class imbalance-related factors $\text{Imb}_1 = \frac{1}{K} \sum_{k=1}^K \sqrt{\frac{1}{\tau_k}},~\text{Imb}_2 = \sqrt{\frac{1}{K} \sum_{k=1}^K \frac{1}{\tau_k}},~\text{Imb}_3 = \frac{1}{\tau_S^*}$ and $\text{Imb}_4 = \frac{1}{\tau_S^*}\left (\frac{1}{K} \sum_{k=1}^K \sqrt{\frac{1}{\tau_k}} \right )$, respectively.}
\label{tab:datasets}
\vskip 0.15in
\begin{center}
\begin{small}
\begin{tabular}{lcccccccc}
\toprule
Dataset & \#Instance & \#Feature & \#Label & Domain & $\text{Imb}_1$ & $\text{Imb}_2$ & $\text{Imb}_3$ & $\text{Imb}_4$ \\
\midrule
    CAL500 & $502$ & $68$ & $174$ & music & $4.2$ & $4.8$ & $100.4$ & $421.1$ \\
    emotions & $593$ & $72$ & $6$ & music & $1.8$ & $1.8$ & $4.0$ & $7.3$ \\
    image & $2000$ & $294$ & $5$ & images & $2.0$ & $2.0$ & $4.9$ & $9.9$ \\
    scene & $2407$ & $294$ & $6$ & images & $2.4$ & $2.4$ & $6.6$ & $15.7$ \\
    yeast & $2417$ & $103$ & $14$ & biology & $2.6$ & $3.2$ & $71.1$ & $188.4$ \\
    enron & $1702$ & $1001$ & $53$ & text & $9.1$ & $11.7$ & $1702$ & $15566$ \\
    rcv1-s1 & $6000$ & $944$ & $101$ & text & $11.8$ & $15.4$ & $3000$ & $35267$ \\
    bibtex & $7395$ & $1836$ & $159$ & text & $9.2$ & $9.4$ & $7395$ & $1332$ \\
    corel5k & $5000$ & $499$ & $374$ & images & $23.4$ & $29.1$ & $5000$ & $117000$ \\
    delicious & $16105$ & $500$ & $983$ & text(web) & $12.2$ & $13.3$ & $766.9$ & $9344$ \\
\bottomrule
\end{tabular}
\end{small}
\end{center}
\vskip -0.1in
\end{table*}

\begin{figure*}[ht]
    \centering
    \begin{subfigure}[CAL500]
        {\includegraphics[scale=0.35]{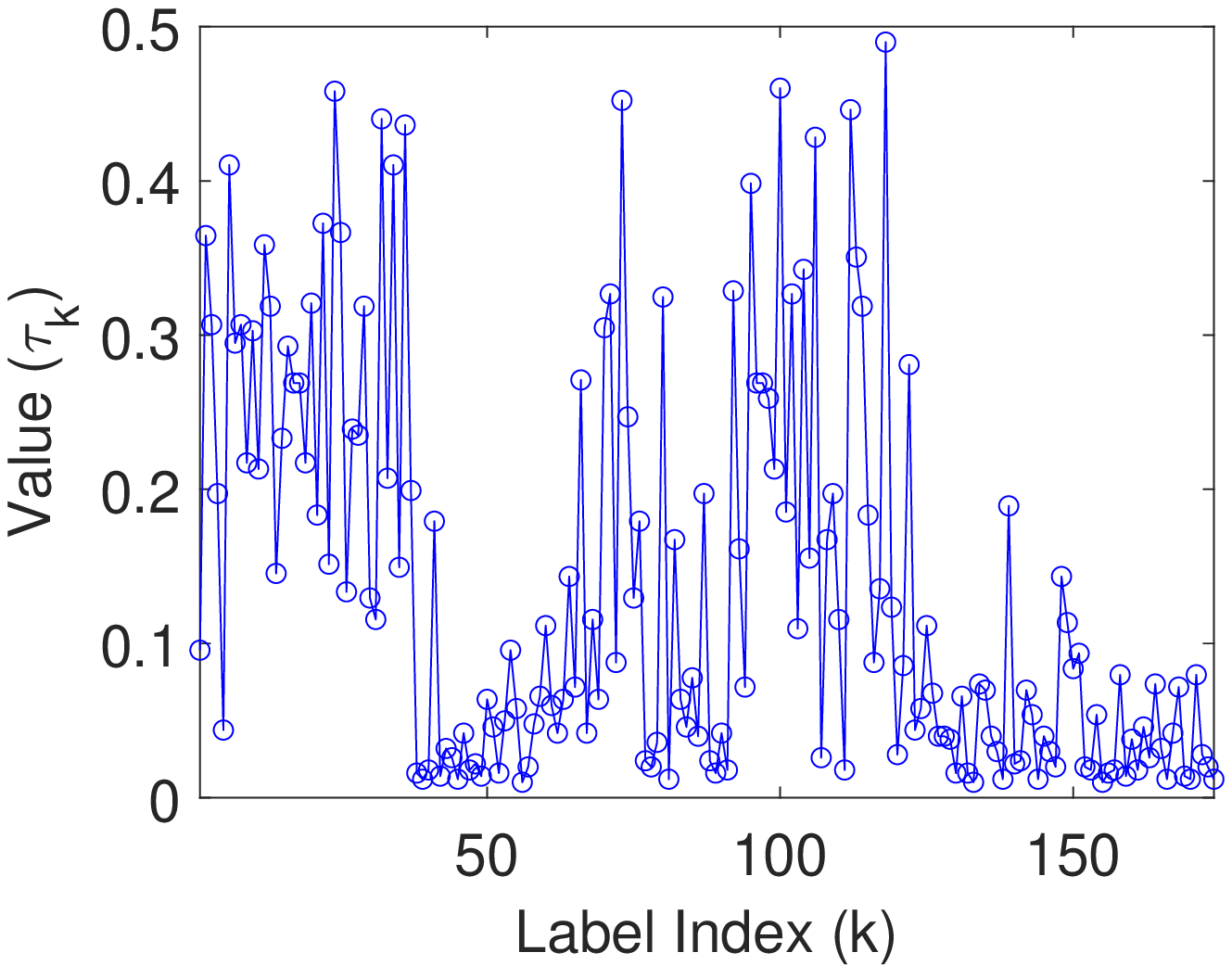}}
    \end{subfigure}
    \begin{subfigure}[image]
        {\includegraphics[scale=0.35]{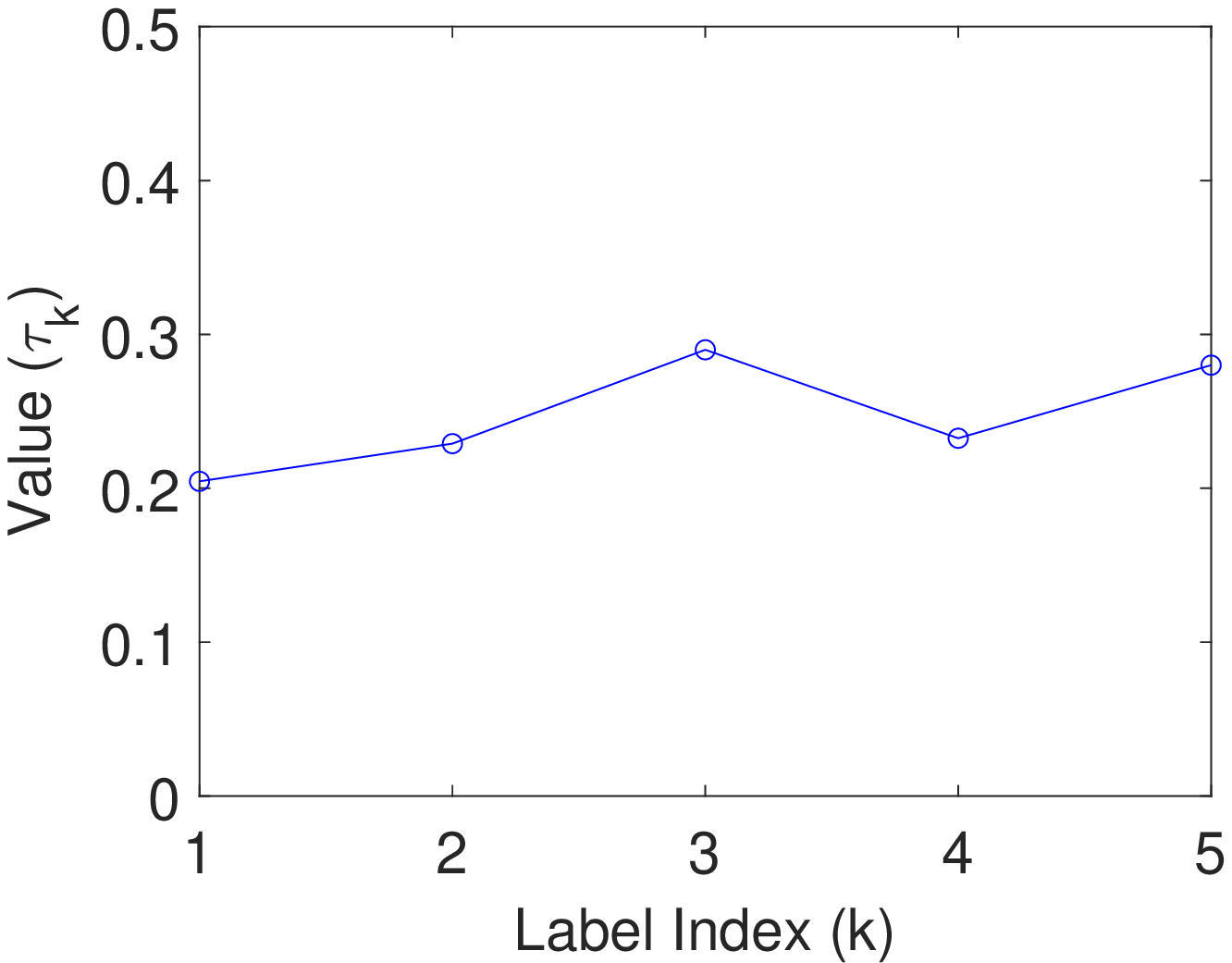}}
    \end{subfigure}
    \begin{subfigure}[delicious]
        {\includegraphics[scale=0.35]{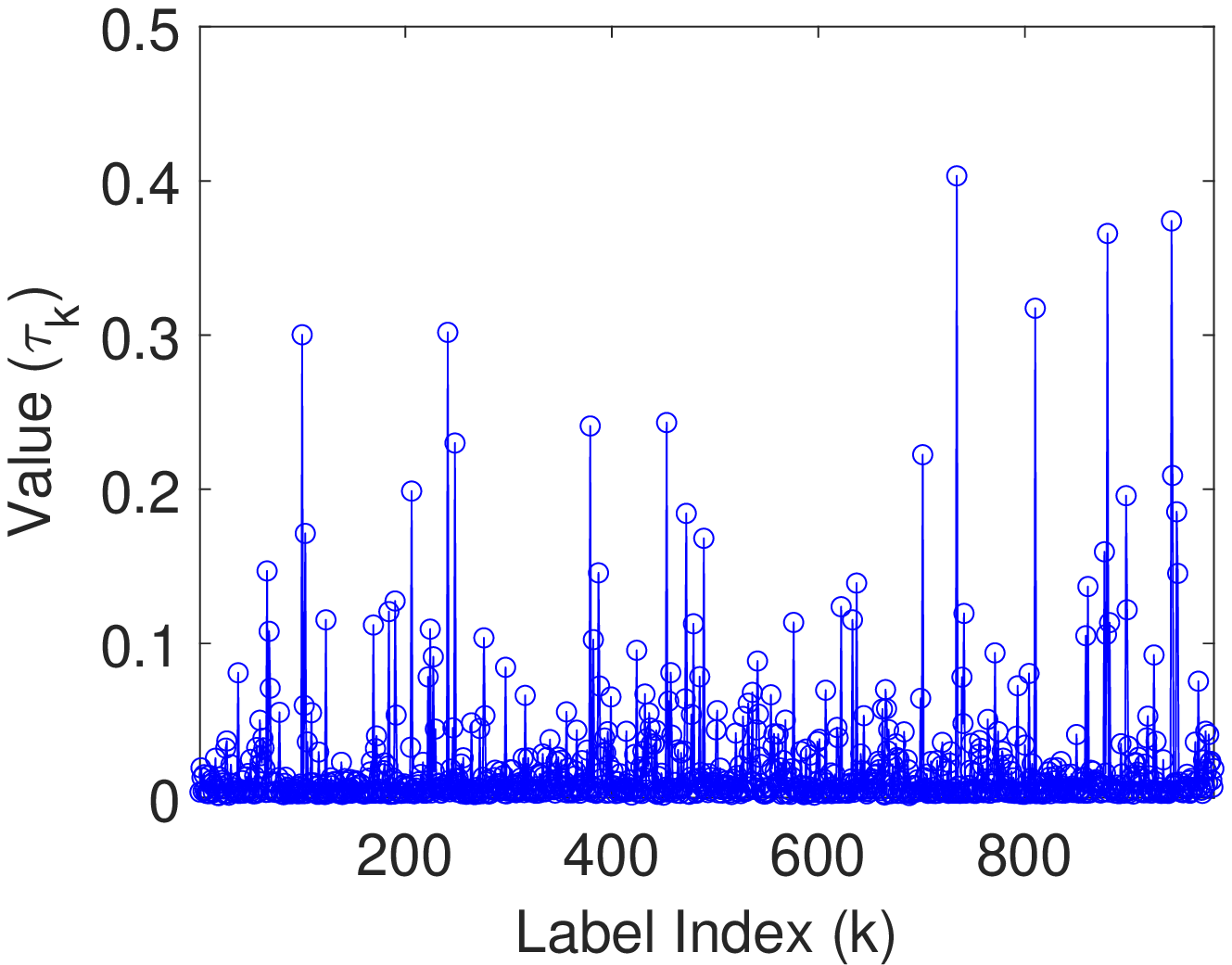}}
    \end{subfigure}
    \vspace{-.3cm}
    \caption{Illustration of the label-wise class imbalance of three representative datasets.}
    \label{fig:imbalance_benchmarks_part}
\end{figure*}

As a theoretical work, the primary goal of experiments is to verify our theory findings rather than illustrate the superior performance of the proposed method. Therefore, we evaluate the aforementioned three learning algorithms in Section~\ref{sec:learning_algorithm} in terms of Macro-AUC on $10$ widely-used benchmark datasets with various domains and sizes of labels and data. The detailed statistics of the datasets are summarized in Table~\ref{tab:datasets}, including four label-wise class imbalance-related factors.\footnote{These datasets can be downloaded from~\url{http://mulan.sourceforge.net/datasets-mlc.html} and~\url{http://palm.seu.edu.cn/zhangml/}.} Besides, the label-wise class imbalance levels of three representative datasets are illustrated in Figure~\ref{fig:imbalance_benchmarks_part}. (See Figure~\ref{fig:imbalance_benchmarks_full} in Appendix~\ref{sec:app_label_wise_class_imbalance_illustration} for all datasets).
For all algorithms, we take linear models with the base logistic loss for simplicity and fair comparison. Besides, we utilize the same efficient stochastic optimization algorithm (i.e., SVRG-BB~\cite{tan2016barzilai}) to solve these convex optimization problems. Moreover, we search the hyper-parameter $\lambda$ for all algorithms on all datasets in a wide range of $\{10^{-6}, 10^{-5}, \dots, 10^2\}$ using $3$-fold cross-validation.\footnote{Our code is available at~\url{https://github.com/GuoqiangWoodrowWu/Macro-AUC-Theory}}

The experimental results are summarized in Table~\ref{tab:benchmark_results}. Overall, we can observe that algorithms $\mathcal{A}^{pa}$ and $\mathcal{A}^{u_2}$ performs better than the algorithm $\mathcal{A}^{u_1}$, which confirms our theoretical results that $\mathcal{A}^{pa}$ and $\mathcal{A}^{u_2}$ have better learning guarantees w.r.t. the label-wise class imbalance than $\mathcal{A}^{u_1}$. Besides, $\mathcal{A}^{pa}$ performs comparably to $\mathcal{A}^{u_2}$, which also verifies our theoretical results that they share the learning guarantee w.r.t. the label-wise class imbalance.

\begin{table}[h!]
\renewcommand\tabcolsep{3.5pt}
\scriptsize
\caption{Macro-AUC ($\text{mean~} \pm \text{~std}$, the symbol $.$ means $0.$) of all three algorithms on benchmark datasets. On each dataset, the top two algorithms are highlighted in bold and the top one is labeled with $^{\dagger}$. Besides, ``-'' means that $\mathcal{A}^{pa}$ takes more than one week by using a $16$-core CPU server on the corresponding datasets.}
\label{tab:benchmark_results}
\begin{center}
\begin{small}
\begin{tabular}{lccc}
\toprule
Dataset & $\mathcal{A}^{pa}$ & $\mathcal{A}^{u_1}$ & $\mathcal{A}^{u_2}$ \\
\midrule
CAL500 & $\bf .5735 \pm .0186^{\dagger}$ & $.5571 \pm .0102$ & $\bf .5717 \pm .0177$ \\
emotions & $\bf .8372 \pm .0172^{\dagger}$ & $.8346 \pm .0223$ & $\bf .8348 \pm .0189$ \\
image & $\bf .8383 \pm .0073^{\dagger}$ & $\bf .8359 \pm .0121$ & $.8314 \pm .0094$ \\
scene & $\bf .9319 \pm .0013^{\dagger}$ & $.9271 \pm .0067$ & $\bf .9285 \pm .0035$ \\
yeast & $\bf .6872 \pm .0100$ & $.6862 \pm .0064$ & $\bf .6892 \pm .0088^{\dagger}$ \\
enron & $\bf .7211 \pm .0320$ & $.6908 \pm .0105$ & $\bf .7356 \pm .0121^{\dagger}$ \\
rcv1-s1 & - & $.8585 \pm .0204$ & $\bf .9097 \pm .0068^{\dagger}$ \\
bibtex & - & $.8693 \pm .0156$ & $\bf .9299 \pm .0034^{\dagger}$ \\
corel5k & - & $.5703 \pm .0092$ & $\bf .6645 \pm .0253^{\dagger}$ \\
delicious & - & $.7633 \pm .0020$ & $\bf .8044 \pm .0040^{\dagger}$ \\
\bottomrule
\end{tabular}
\end{small}
\end{center}
\end{table}

Further, from Table~\ref{tab:datasets} and~\ref{tab:benchmark_results}, we can carefully study the effects of the label-wise class imbalance on the performance. Recall that the learning guarantees of $\mathcal{A}^{pa}$ and $\mathcal{A}^{u_2}$ both depends on the factor $\text{Imb}_1$, while the one of $\mathcal{A}^{u_1}$ depends on the factor $\text{Imb}_4$. For the datasets CAL500, enron, rcv-s1, bibtex, corel5k and delicious, factors $\text{Imb}_1$ and $\text{Imb}_4$ have a large order gap (or equivalently $\text{Imb}_3$ is large), and $\mathcal{A}^{u_2}$ (or $\mathcal{A}^{pa}$) performs significantly better than $\mathcal{A}^{u_1}$. In contrast, for the datasets emotions, image, and scene, factors $\text{Imb}_1$ and $\text{Imb}_4$ have a small gap (or equivalently $\text{Imb}_3$ is small), and $\mathcal{A}^{u_2}$ (or $\mathcal{A}^{pa}$) performs slightly better than or is nearly comparable to $\mathcal{A}^{u_1}$. This also confirms our theoretical findings of these algorithms on the label-wise class imbalance.

Furthermore, similarly to previous theoretical results~\cite{wu2021rethinking} for the Ranking Loss measure in MLC, our generalization upper bound absolute values might not reflect the true generalization error reasonably well (i.e., bigger than $1$). However, they can still offer valuable insight into these learning algorithms under the same analytical framework. (See Table~\ref{tab:empirical_bounds} in Appendix~\ref{sec:app_absolute_value_bounds} for details). Advanced techniques (e.g., local Radermacher-type complexity) can refine the results, left as future work.

\section{Conclusion}

Towards understanding the generalization of Macro-AUC in MLC, this paper takes an initial step by analyzing the generalization bounds of the algorithms with various surrogates including the widely-used univariate one. Our results show that the label-wise class imbalance of the dataset plays a critical role in these bounds. The algorithms with the proposed pairwise and reweighted univariate loss have better learning guarantees than the original univariate-based algorithm, which probably implies their superior performance. Experimental results also confirm our theoretical findings. 

\textbf{Social Impact:} As a theoretical research, this work will help understand and potentially develop better algorithms for multi-label learning, while without explicit negative consequences to society.

\section*{Acknowledgements}

This work was supported by NSF of China (Nos. 62206159, 62076145); Shandong Provincial Natural Science 
Foundation (Nos. ZR2022QF117, ZR2021ZD15); Beijing Outstanding Young Scientist Program (NO. BJJWZYJH012019100020098); the Fundamental Research Funds of Shandong University; Major Innovation \& Planning Interdisciplinary Platform for the ``Double-First Class" Initiative, Renmin University of China; the Fundamental Research Funds for the Central Universities, and the Research Funds of Renmin University of China (22XNKJ13). C. Li was also sponsored by Beijing Nova Program.


\bibliography{reference}
\bibliographystyle{icml2023}

\newpage
\appendix
\onecolumn

\renewcommand{\contentsname}{Contents of Appendix}
\tableofcontents
\addtocontents{toc}{\protect\setcounter{tocdepth}{3}} 

\clearpage


\section{General Techniques}
\label{sec:app_general_techniques}


In this section, we introduce the general techniques, which mainly consist of a new McDiarmid-type concentration inequality and a general generalization bound of learning multiple tasks with graph-dependent examples.

\subsection{A new McDiarmid-type concentration inequality}

\subsubsection{Backgrounds}

First, we introduce the bounded differences property and a lemma for the proof of the subsequent theorem.
\begin{definition}[\textbf{The bounded differences property}~\cite{mcdiarmid1989method}]
    Let $x_1, x_2, \dots, x_m \in \mathcal{X}$, and function $f: \mathcal{X}^m \rightarrow \sR$. Then, $f$ is said to have bounded differences property if there exist $c_1,\dots,c_m > 0$ such that
    \begin{align*}
        |f(x_1,\dots,x_i,\dots,x_m) - f(x_1,\dots,x_i^{\prime},\dots,x_m)| \leq c_i,
    \end{align*}
    for all $i \in [m]$ and any points $x_1,\dots,x_m, x_i^{\prime} \in \mathcal{X}$.
\end{definition}

\begin{lemma}[\cite{mcdiarmid1989method}]
\label{lemma_mcdiarmid1}
    Let $\mX = (X_1,\dots, X_m) \in \mathcal{X}^m$ be a vector of $m$ independent random variables and function $f: \mathcal{X}^m \rightarrow \sR$ satisfies the bounded differences property with $c_i$ ($i \in [m]$), then for any $s > 0$,
    \begin{align*}
        \eE[\exp (s(f(\mX) - \eE [f(\mX)]))] \leq \exp \left( \frac{s^2 \sum_{i \in [m]} c_i^2}{8} \right) .
    \end{align*}
\end{lemma}

Here we introduce some necessary notions of graph theory in this paper, and we refer readers to~\citep{janson2004large, amini2015learning} and recent survey~\cite{zhang2022generalization}.

Given a graph $G = (V, E)$, we introduce the following notions.
\begin{definition}[\textbf{Fractional independent vertex cover, and fractional chromatic number}~\cite{zhang2022generalization}]
~
    \begin{enumerate}[(1)]
        \item A family $\{ (F_j, \omega_j) \}_j$ of pairs $(F_j, \omega_j)$, where $F_j \subseteq V(G)$ and $\omega_j \in (0,1]$ is a \textbf{fractional vertex cover} of $G$ if $\sum_{j:v \in F_j} \omega_j = 1$ for every $v \in V(G)$.
        \item An \textbf{independent set} of $G$ is a set of vertices in $G$ such that no two them are adjacent. The set of independent sets of $G$ is denoted by $\mathcal{I}(G)$.
        \item A \textbf{fractional independent vertex cover} $\{ (I_j, \omega_j) \}_j$ of $G$ is fractional vertex cover such that $I_j \in \mathcal{I}(G)$ for every $j$.
        \item A \textbf{fractional coloring} of a graph $G$ is a mapping $g$ from $\mathcal{I}(G)$ to $(0,1]$ such that $\sum_{I \in \mathcal{I}(G): v \in I} g(I) \geq 1$ for every vertex $v \in V(G)$. The \textbf{fractional chromatic number} $\chi_f(G)$ is the minimum of the value $\sum_{I \in \mathcal{I}(G)} g(I)$ over fractional colorings of $G$.
        
        Note that the fractional chromatic number $\chi_f(G)$ of a graph $G$ is the minimum of $\sum_{j} \omega_j$ over all fractional independent vertex covers $\{ (I_j, \omega_j) \}_j$ of $G$.
    \end{enumerate}
\end{definition}

Next, we introduce the notion of dependency graph as follows.
\begin{definition}[\textbf{Dependency graph}~\cite{janson2004large}]
    An undirected graph $G = (V, E)$ is called a dependency graph of a random vector $\mX = (X_1,\dots, X_m)$ if \begin{enumerate}[(1)]
        \item $V(G) = [m]$.
        \item For all disjoint $I, J \in [m]$, if $I, J$ are not adjacent in $G$, then random variables $\{X_i\}_{i \in I}$ and $\{X_j\}_{j \in J}$ are independent.
    \end{enumerate}
    Then, we say that random vector $\mX$ is $G$-dependent with a dependency graph $G$.
\end{definition}
An important property of the dependency graph, combined with the notion of fractional independent covers, is Janson's decomposition property~\cite{janson2004large}. Specifically, suppose interdependent random variables $(X_i)_{i \in [m]}$ is $G$-dependent with a dependency graph $G$, and $\{ (I_j, \omega_j) \}_{j \in [J]}$ is a fractional independent vertex cover of $G$. Then, we can decompose the sum of interdependent variables into a weighted sum of sums of independent variables, i.e.,
\begin{align}
\label{eq:app_jansen_decomposition}
    \sum_{i=1}^m X_i = \sum_{i=1}^m \sum_{j=1}^J \omega_j [\![ i \in I_j ]\!] X_i = \sum_{j=1}^J \omega_j \sum_{i \in I_j} X_i .
\end{align}

\subsubsection{Proof of the new McDiarmid-type concentration inequality}
Here we propose a new and more general McDiarmid-type inequality as follows, which mainly follows the work~\cite{janson2004large, usunier2005generalization,amini2015learning} and we refer to a recent related survey~\cite{zhang2022generalization}.

\begin{restatable}[\textbf{A new and more general McDiarmid-type inequality}]
    {theorem}{NewMcDiarmid}
\label{thm:new_mcdiarmid}
    Let $\mX_1 = (\vx_{11},\dots,\vx_{1m_1}) \in \mathcal{X}^{m_1}$, \dots, $\mX_K = (\vx_{K1},\dots,\vx_{Km_K}) \in \mathcal{X}^{m_K}$ be vectors of random variables and $\mX$ denote $(\mX_1,\dots, \mX_K) = (\vx_{11},\dots, \vx_{Km_K})$ for convenience. Let $f_1: \mathcal{X}^{m_1} \rightarrow \sR$, \dots, $f_K: \mathcal{X}^{m_K} \rightarrow \sR$ and $f: \mathcal{X}^{m} \rightarrow \sR$ be functions with $\sum_{k=1}^K m_k = m$. Assume each $\mX_k$ ($k \in [K]$) is $G_k$-dependent with a dependency graph $G_k$.\footnote{Note that here we only make the dependency assumptions within each $\mathbf{X}_k$ but have no assumptions between different $\mathbf{X}_k$s, where $\mathbf{X}_k$s can be independent or dependent, regardless of independence.} Besides, assume the function $f$ satisfies the following constraints: 
    \begin{enumerate}[(1)]
    \setlength\itemsep{-3pt}
        \item $f(\mX) = \sum_{k \in [K]} f_k(\mX_k)$;
        \item $f_k(\mX_k)$ has the decomposability constraint with the bounded difference property w.r.t. the graph $G_k$, i.e., for all $\vx_k \in \mathcal{X}^{m_k}$ and the minimal fractional independent vertex covers $\{ (I_{kj}, \omega_{kj})\}_{j \in [J_k]}$ of $G_k$, there exists functions $\{ f_{kj}: \mathcal{X}^{|I_{kj}|} \rightarrow \sR\}_{j \in [J_k]}$ such that $f_{kj}$ satisfies the bounded difference property with $c_{ki}$ ($i \in I_{kj}$) and 
        \begin{align*}
            f_k(\vx_k) = \sum_{j \in [J_k]} \omega_{kj} f_{kj} (\vx_{I_{kj}}) ,
        \end{align*}
        where $\vx_{I_{kj}}$ denotes $(\vx_{ki})_{i \in I_{kj}}$.
    \end{enumerate}
    Then, for any $t > 0$,
    \begin{align*}
        \pP (f(\mX) - & \eE[f(\mX)] \geq t)  \leq \\
        & \exp \left(- \frac{2 t^2}{K \sum_{k \in [K]} \left( \chi_{f}(G_k) \sum_{i\in [m_k]} c_{ki}^2 \right) } \right) , 
    \end{align*}
    where $\chi_{f}(G_k)$ is the fractional chromatic number of $G_k$.
    
\end{restatable}
\begin{remark*}
    The McDiarmid-type inequality in prior work~\cite{usunier2005generalization,amini2015learning} can be viewed as a special case of the above one by setting $K=1$. Thus, our proposed new McDiarmid-type inequality is more general.
\end{remark*}

\begin{proof}
    Following the Cram{\'e}r-Chernoff method~\cite{boucheron2013concentration}, we have for any $s > 0$ and $t > 0$,
    \begin{align}
    \label{eq:theorem_McDiarmid_tmp1}
        \pP (f(\mX) - \eE [f(\mX)] \geq t) \leq e^{-st} \eE[\exp (s(f(\mX) - \eE [f(\mX)]))] .
    \end{align}
    For the dependency graph, the $k$-th sub-graph $G_k$ has $m_k$ vertexes. Further,
    let $I_{k}$ be the vertex set of $G_k$ and $\{ (I_{kj}, \omega_{kj})\}_{j \in [J_k]}$ be a minimal fractional independent vertex cover of $G_k$ with $\sum_{j \in [J_k]} \omega_{kj} = \chi_{f} (G_k)$. Utilizing the decomposition property $f(\vx) = \sum_{k \in [K]} f_k(\vx_k) = \sum_{k \in [K]} \sum_{j \in [J_k]} \omega_{kj} f_{kj} (I_{kj})$ where $\vx_{I_{kj}}$ is denoted by $I_{kj}$ for notation simplicity, 
    we have for the expectation term on the right-hand side of the above inequality:
    \begin{align*}
        \eE[\exp (s(f(\mX) - \eE [f(\mX)]))] = \eE \left [ \exp \left(\sum_{k \in [K]} \sum_{j \in [J_k]} s \omega_{kj} \left( f_{kj} (I_{kj}) - \eE f_{kj} (I_{kj}) \right) \right) \right] .
    \end{align*}
    Let $\{p_1, p_2,..., p_K\}$ be any set of $K$ strictly positive real numbers that sum to $1$. Similarly, for each $k \in [K]$, let $\{q_{k1}, q_{k2},..., q_{k J_k}\}$ be any set of $J_k$ strictly positive real numbers that sum to $1$. Then, based on the convexity of the exponential function, we can have the following:
    \begin{align*}
        \eE[\exp (s(f(\mX) - \eE [f(\mX)]))] & = \eE \left [ \exp \left( \sum_{k \in [K]} \sum_{j \in [J_k]} s \omega_{kj} \left( f_{kj} (I_{kj}) - \eE f_{kj} (I_{kj}) \right) \right) \right] & \\
        & = \eE \left [ \exp \left(\sum_{k \in [K]} p_k \sum_{j \in [J_k]} \frac{s \omega_{kj}}{p_k} \left( f_{kj} (I_{kj}) - \eE f_{kj} (I_{kj}) \right) \right) \right] & (\text{definition of~} p_k)\\
        & \leq \eE \left [ \sum_{k \in [K]} p_k \exp \left( \sum_{j \in [J_k]} \frac{s \omega_{kj}}{p_k} \left( f_{kj} (I_{kj}) - \eE f_{kj} (I_{kj}) \right) \right) \right] & (\text{Jensen's inequality}) \\
        & = \eE \left [ \sum_{k \in [K]} p_k \exp \left( \sum_{j \in [J_k]} q_{kj} \frac{s \omega_{kj}}{p_k q_{kj}} \left( f_{kj} (I_{kj}) - \eE f_{kj} (I_{kj}) \right) \right) \right] & (\text{definition of~} p_{kj})\\
        & \leq \eE \left [ \sum_{k \in [K]} p_k \sum_{j \in [J_k]} q_{kj} \exp \left(  \frac{s \omega_{kj}}{p_k q_{kj}} \left( f_{kj} (I_{kj}) - \eE f_{kj} (I_{kj}) \right) \right) \right] & (\text{Jensen's inequality}) \\
        & = \sum_{k \in [K]} p_k \underbrace{\sum_{j \in [J_k]} q_{kj} \eE \left [ \exp \left( \frac{s \omega_{kj}}{p_k q_{kj}} \left( f_{kj} (I_{kj}) - \eE f_{kj} (I_{kj}) \right) \right) \right]}_{\overset{def}{=} \clubsuit_k} & (\text{linearity of expectation}) .
    \end{align*}
    Here we can observe that for the summation term $\clubsuit_k$, the random variables associated with each term $j \in [J_k]$ are independent. Thus, applying the Lemma~\ref{lemma_mcdiarmid1}, we can get
    \begin{align*}
        \clubsuit_k = \sum_{j \in [J_k]} q_{kj} \eE \left [ \exp \left( \frac{s \omega_{kj}}{p_k q_{kj}} \left( f_{kj} (I_{kj}) - \eE f_{kj} (I_{kj}) \right) \right) \right] \leq \sum_{j \in [J_k]} q_{kj} \exp \left( \frac{s^2 \omega_{kj}^2}{8 p_k^2 q_{kj}^2} \sum_{i \in I_{kj}} c_{ki}^2 \right) .
    \end{align*}
    By rearranging terms in the exponential of right hand side of the inequality above and by setting
    \begin{align*}
        q_{kj} = \frac{\omega_{kj} \sqrt{\sum_{i \in I_{kj}} c_{ki}^2}}{\sum_{j \in [J_k]} \left( \omega_{kj} \sqrt{\sum_{i \in I_{kj}} c_{ki}^2} \right)} ,
    \end{align*}
    we have:
    \begin{align*}
        \sum_{j \in [J_k]} q_{kj} \exp \left( \frac{s^2 \omega_{kj}^2}{8 p_k^2 q_{kj}^2} \sum_{i \in I_{kj}} c_{ki}^2 \right) & = \sum_{j \in [J_k]} q_{kj} \exp \left( \frac{s^2}{8 p_k^2} \left( \sum_{j \in [J_k]} \left( \omega_{kj} \sqrt{\sum_{i \in I_{kj}} c_{ki}^2} \right) \right)^2 \right) \\
        & = \exp \left( \frac{s^2}{8 p_k^2} \left( \sum_{j \in [J_k]} \left( \omega_{kj} \sqrt{\sum_{i \in I_{kj}} c_{ki}^2} \right) \right)^2 \right) \qquad (\sum_{j \in [m_k]} q_{kj} = 1) .
    \end{align*}
    
    Till now, we have the following:
    \begin{align*}
        \eE[\exp (s(f(\mX) - \eE [f(\mX)]))] \leq \sum_{k \in [K]} p_k \clubsuit_k \leq \sum_{k \in [K]} p_k \exp \left( \frac{s^2}{8 p_k^2} \left( \sum_{j \in [J_k]} \left( \omega_{kj} \sqrt{\sum_{i \in I_{kj}} c_{ki}^2} \right) \right)^2 \right) .
    \end{align*}
    
    Next, similarly to the above proof idea w.r.t. the $q_{kj}$, we set $p_k$ as follows:
    \begin{align*}
        p_k = \frac{\sum_{j \in [J_k]} \left( \omega_{kj} \sqrt{\sum_{i \in I_{kj}} c_{ki}^2} \right) }{\sum_{k \in [K]} \sum_{j \in [J_k]} \left( \omega_{kj} \sqrt{\sum_{i \in I_{kj}} c_{ki}^2} \right)} . 
    \end{align*}
    Then, it comes:
    \begin{align*}
        \sum_{k \in [K]} p_k \exp \left( \frac{s^2}{8 p_k^2} \left( \sum_{j \in [J_k]} \left( \omega_{kj} \sqrt{\sum_{i \in I_{kj}} c_{ki}^2} \right) \right)^2 \right) & = \sum_{k \in [K]} p_k \exp \left( \frac{s^2}{8} \left( \sum_{k \in [K]} \sum_{j \in [J_k]} \left( \omega_{kj} \sqrt{\sum_{i \in I_{kj}} c_{ki}^2} \right) \right)^2 \right) \\
        & = \exp \left( \frac{s^2}{8} \left( \sum_{k \in [K]} \sum_{j \in [J_k]} \left( \omega_{kj} \sqrt{\sum_{i \in I_{kj}} c_{ki}^2} \right) \right)^2 \right) \quad (\sum_{k \in [K]} p_k = 1) 
        \\
        & \overset{\circled{1}}{\leq} \exp \left( \frac{s^2 K}{8} \sum_{k \in [K]} \left( \sum_{j \in [J_k]}  \omega_{kj} \sqrt{\sum_{i \in I_{kj}} c_{ki}^2} \right)^2 \right) \\
        & = \exp \left( \frac{s^2 K}{8} \sum_{k \in [K]} \left( \sum_{j \in [J_k]} \left( \sqrt{\omega_{kj}} \sqrt{\omega_{kj} \sum_{i \in I_{kj}} c_{ki}^2} \right) \right)^2 \right)  \\
        & \overset{\circled{2}}{\leq} \exp \left( \frac{s^2 K}{8} \sum_{k \in [K]} \left( \sum_{j \in [J_k]} \omega_{kj} \right) \left( \sum_{j \in [J_k]} \omega_{kj} \sum_{i \in I_{kj}} c_{ki}^2 \right) \right) \\ 
        & \overset{\circled{3}}{=} \exp \left( \frac{s^2 K}{8} \sum_{k \in [K]} \chi_{f} (G_k) \left( \sum_{i \in [m_k]} c_{ki}^2 \right) \right) .
    \end{align*}
    For $\circled{1}$, it is based on the inequality $\left(\sum_{i=1}^n a_i \right)^2 \leq n \sum_{i=1}^n a_i^2$. For $\circled{2}$, it is due to the Cauchy-Schwarz inequality. For $\circled{3}$, it is due to the definition of the fractional chromatic number, i.e., $\sum_{j \in [J_k]} \omega_{kj} = \chi_{f} (G_k)$, 
    and the decomposition property of fractional independent vertex covers of dependency graph~\cite{janson2004large}, i.e., 
    for a fractional independent vertex cover $\{ (I_{kj}, \omega_{kj})\}_{j \in [J_k]}$ of $G_k$, then the sum of interdependent variables can be decomposed into a weighted sum of sums of independent variables as follows:
    \begin{align*}
        \sum_{i=1}^{m_k} \vx_{ki} = \sum_{i=1}^{m_k} \sum_{j=1}^{J_k} \omega_{kj} [\![ i \in I_{kj} ]\!] \vx_{ki} = \sum_{j=1}^{J_k} \omega_{kj} \sum_{i \in I_{kj}} \vx_{ki} .
    \end{align*}
    Since $\mathbf{X}_{k} = [\mathbf{x}_{k1}, ..., \mathbf{x}_{k m_k}]$ is a random vector, we can take the specific values to get the equation $\circled{3}$ based on the inequality $\circled{2}$. Specifically, if we take $\mathbf{x}_{ki} = c_{ki}^2$ for each $i \in [m_k]$, then we can get 
    \begin{align*}
        \sum_{i=1}^{m_k} c_{ki}^2 = \sum_{j=1}^{J_k} \omega_{kj} \sum_{i \in I_{kj}} c_{ki}^2 .
    \end{align*}
    Thus, we have obtained 
    \begin{align*}
        \eE[\exp (s(f(\mX) - \eE [f(\mX)]))] \leq \exp \left( \frac{s^2 K}{8} \sum_{k \in [K]} \chi_{f} (G_k) \left( \sum_{i \in [m_k]} c_{ki}^2 \right) \right).
    \end{align*}
    Combining the inequality~\eqref{eq:theorem_McDiarmid_tmp1}, we can get
    \begin{align*}
        \pP (f(\mX) - \eE [f(\mX)] \geq t) \leq \exp \left( -st + \frac{s^2 K}{8} \sum_{k \in [K]} \chi_{f} (G_k) \left( \sum_{i \in [m_k]} c_{ki}^2 \right) \right).
    \end{align*}
    We can obtain the final result by minimizing the right-hand side of the above inequality over $s$.
\end{proof}

\subsection{Learning multiple tasks with graph-dependent examples}

\subsubsection{Problem setting}
\label{sec:app_problem_setting}
    Here we consider learning with multiple tasks where each task might contain dependent training examples and the dependency relationship is characterized by a dependency graph. Formally, given a training dataset $\tilde{S} = \{(\tilde{\vx}, \ty)\}_{i=1}^{m}$ that is composed of $K$ blocks (or tasks), i.e., $\tilde{S} = (\tilde{S}_1,\dots,\tilde{S}_K)$ with each $\tilde{S}_k = \{(\tilde{\vx}_{ki}, \ty_{ki})\}_{i=1}^{m_k}$ drawn from the distribution $D_k$ ($k \in [K]$) over $\widetilde{\mathcal{X}} \times \widetilde{\mathcal{Y}}$ with a dependency graph $G_k$ and $\sum_{k \in [K]} m_k = m$. 
    The goal is to learn a mapping $\tilde{h} = (\tilde{h}_1,\dots,\tilde{h}_K)$, where $\tilde{h}_k: \widetilde{\mathcal{X}} \rightarrow \widetilde{\mathcal{Y}}$ for each $k \in [K]$.
    
    Let $\widetilde{\mathcal{F}} = \left \{ \tilde{f} = (\tilde{f}_1,\dots,\tilde{f}_K)~|~ \tilde{f}_k: \widetilde{\mathcal{X}} \rightarrow \widehat{\mathcal{Y}}, k \in [K] \right \}$ be the hypothesis space, and denote $\widetilde{\mathcal{F}}_k = \left \{\tf_k ~|~ \tilde{f}_k: \widetilde{\mathcal{X}} \rightarrow \widehat{\mathcal{Y}} \right \}$ for each $k \in [K]$. Consider a loss function $L: \widetilde{\mathcal{X}} \times \widetilde{\mathcal{Y}} \times \widetilde{\mathcal{F}}_k \rightarrow \sR_+$. For a hypothesis $\tf \in \wtF$ and a training set $\tS$, the empirical risk of $\tf$ is defined as
    \begin{align*}
        \widehat{R}_{\tS}(\tilde{f}) = \frac{1}{K} \sum_{k=1}^K \frac{1}{m_k} \sum_{i=1}^{m_k} L(\tilde{\vx}_{ki}, \ty_{ki}, \tilde{f}_k) ,
    \end{align*}
    and the generalization (or expected) risk is defined as
    \begin{align}
    \label{eq:app_general_risk1}
        R(\tilde{f}) = \eE_{\tS} \left[ \widehat{R}_{\tS}(\tilde{f}) \right].
    \end{align}
    Note that we do not define the generalization risk as the following usual form
    \begin{align}
    \label{eq:app_general_risk2}
        \frac{1}{K} \sum_{k=1}^K \eE_{(\tilde{\vx}, \tilde{y}) \sim D_k} \left[ L(\tilde{\vx}, \ty, \tilde{f}_k) \right] .
    \end{align}
    This is because the definition in Eq.\eqref{eq:app_general_risk1} is more general than Eq.\eqref{eq:app_general_risk2}. Specifically, Eq.\eqref{eq:app_general_risk1} can cover the loss function dependent on the training set $\tS$ while Eq.\eqref{eq:app_general_risk2} cannot. Besides, they are equal for certain losses independent of $\tS$.

\subsubsection{The fractional Rademacher complexity of the loss space}

Here we give the definition of the fractional Rademacher complexity of the loss space as follows.
\begin{definition}[\textbf{The fractional Rademacher complexity of the loss space}]
For each $k \in [K]$, let $\{ (I_{kj}, \omega_{kj})\}_{j \in [J_k]}$ be a fractional independent vertex cover of the dependence graph $G_k$ constructed over $\tS_k$ with $\sum_{j \in [J_k]} \omega_{kj} = \chi_{f} (G_k)$. Let $\widetilde{\mathcal{F}} = \left \{ \tilde{f} = (\tilde{f}_1,\dots,\tilde{f}_K)~|~ \tilde{f}_k: \widetilde{\mathcal{X}} \rightarrow \widehat{\mathcal{Y}}, k \in [K] \right \}$ be the hypothesis space. Then, the empirical fractional Rademacher complexity of $\wtF$ given $\tS$ is defined by
    \begin{align*}
        \widehat{\mathfrak{R}}_{\tS}^*(L \circ \widetilde{\mathcal{F}}) = \frac{1}{K} \sum_{k=1}^K \eE_{\boldsymbol{\sigma}} \left[ \frac{1}{m_k} \sum_{j \in [J_k]} \omega_{kj} \sup_{\tf \in \wtF} \left( \sum_{i \in I_{kj}} \sigma_{ki} L(\tilde{\vx}_{ki}, \ty_{ki}, \tilde{f}_k) \right) \right],
    \end{align*}
where $\boldsymbol{\sigma} = (\sigma_{ki})_{k \in [K], i \in [m_k]}$ denotes $m$ independent Rademacher variables, that is, $\pP(\sigma_{ki} = +1) = \pP(\sigma_{ki} = -1) = 1/2$ for all variables. Furthermore, the fractional Rademacher complexity of $\wtF$ over all samples of size $m$ is defined by
    \begin{align*}
        \mathfrak{R}_m^*(L \circ \wtF) = \eE_{\tS \sim D_{[K]}^m} \left[ \widehat{\mathfrak{R}}_{\tS}^*(L \circ \wtF)  \right],
    \end{align*}
where $\tS \sim D_{[K]}^m$ denotes $\tS_1 \sim D_1^{m_1},\dots, \tS_K \sim D_K^{m_K}$ for simplicity.
\end{definition}

\subsubsection{Proof of the general generalization bound of learning multiple tasks with graph-dependent examples}

Here we give a general generalization bound of learning multiple tasks with graph-dependent examples as follows.


\begin{restatable}[\textbf{A general generalization bound of learning multiple tasks with graph-dependent examples}]
    {theorem}{GeneralBoundMultiTaskGraphDepedent}
\label{thm:general_bound_multiple_tasks_dependent}
Give a sample $\tilde{S} = \{\tilde{S}_1, \dots, \tilde{S}_K\}$ where each $\tilde{S}_{k \in [K]}$ is of size $m_k$ with dependency graph $G_k$ and a loss function $L: \widetilde{\mathcal{X}} \times \widetilde{\mathcal{Y}} \times \widetilde{\mathcal{F}}_k \rightarrow [0, M]$. Then, for any $\delta \in (0, 1)$, with probability at least $1 - \delta$, we have 
    \begin{align}
    \label{eq:general_bound1}
        \forall \tf \in \wtF, \ R(\tf) \leq & \widehat{R}_{\tS} (\tf) + 2 \mathfrak{R}_m^*(L \circ \wtF) \ + \nonumber\\
        & M \sqrt{ \left(\frac{1}{K} \sum_{k=1}^K \frac{\chi_{f}(G_k)}{2 m_k} \right) \log \left(\frac{1}{\delta} \right)} \ ,
    \end{align}
and
    \begin{align}
    \label{eq:general_bound2}
        \forall \tf \in \wtF, \ R(\tf) \leq & \widehat{R}_{\tS} (\tilde{f}) + 2 \widehat{\mathfrak{R}}_{\tS}^*(L \circ \wtF) \ + \nonumber\\
        & 3 M \sqrt{ \left(\frac{1}{K} \sum_{k=1}^K \frac{\chi_{f}(G_k)}{2 m_k} \right) \log \left(\frac{2}{\delta} \right)} \ .
    \end{align}
\end{restatable}

\begin{proof}
    The proof can be divided into three major steps as follows.
    
    \textbf{Step 1: link the supremum of $R(\tf) - \widehat{R}_{\tS} (\tf)$ on $\wtF$ with its expectation.}
    
    For any $\tilde{f} \in \widetilde{\mathcal{F}}$, we have $\widehat{R}(\tilde{f})$ is an unbiased estimator of $R(\tilde{f})$ because the data points in the sample $\tilde{S}_{k}$ are assumed to be $G$-dependent and have the same marginal distribution. Hence considering an independent ghost sample $\tilde{S}^{\prime}$ with the same generation process as $\tilde{S}$, we have
    \begin{align*}
        \sup_{\tf \in \wtF} \left(R(\tf) - \widehat{R}_{\tS} (\tf) \right) = \sup_{\tf \in \wtF} \left( \eE_{\tS^{\prime}} [ \widehat{R}_{\tS^{\prime}} (\tf)] - \widehat{R}_{\tS} (\tf) \right) = \sup_{\tf \in \wtF} \left( \eE_{\tS^{\prime}} \left[ \widehat{R}_{\tS^{\prime}} (\tf) - \widehat{R}_{\tS} (\tf) \right] \right) .
    \end{align*}
    For each $k \in [K]$, let $\{ (I_{kj}, \omega_{kj})\}_{j \in [J_k]}$ be a fractional independent vertex cover of the dependence graph $G_k$ with $\sum_{j \in [J_k]} \omega_{kj} = \chi_{f} (G_k)$.
    Since the supremum of the expectation is lower than the expectation of the supremum, we can have
    \begin{align*}
        \sup_{\tf \in \wtF} \left( \eE_{\tS^{\prime}} \left[ \widehat{R}_{\tS^{\prime}} (\tf) - \widehat{R}_{\tS} (\tf) \right] \right) & \leq  \eE_{\tS^{\prime}} \left[ \sup_{\tf \in \wtF} \left(  \widehat{R}_{\tS^{\prime}} (\tf) - \widehat{R}_{\tS} (\tf) \right) \right] \\
        & = \eE_{\tS^{\prime}} \left[ \sup_{\tf \in \wtF} \left( \frac{1}{K} \sum_{k=1}^K \frac{1}{m_k} \sum_{i=1}^{m_k} \left( L(\tilde{\vx}_{ki}^{\prime}, \ty_{ki}^{\prime}, \tilde{f}_k) - L(\tilde{\vx}_{ki}, \ty_{ki}, \tilde{f}_k) \right) \right) \right] \\
        & \overset{\circled{1}}{=} \eE_{\tS^{\prime}} \left[ \sup_{\tf \in \wtF} \left( \frac{1}{K} \sum_{k=1}^K \sum_{j \in [J_k]} \frac{\omega_{kj}}{m_k} \sum_{i \in I_{kj}} \left( L(\tilde{\vx}_{ki}^{\prime}, \ty_{ki}^{\prime}, \tilde{f}_k) - L(\tilde{\vx}_{ki}, \ty_{ki}, \tilde{f}_k) \right) \right) \right] \\
        & \overset{\circled{2}}{\leq} \frac{1}{K} \sum_{k=1}^K \sum_{j \in [J_k]} \frac{\omega_{kj}}{m_k} \eE_{\tS_k^{\prime}} \left[ \sup_{\tf \in \wtF} \left( \sum_{i \in I_{kj}} \left( L(\tilde{\vx}_{ki}^{\prime}, \ty_{ki}^{\prime}, \tilde{f}_k) - L(\tilde{\vx}_{ki}, \ty_{ki}, \tilde{f}_k) \right) \right) \right] ,
    \end{align*}
    where the inequality $\circled{1}$ is due to the Janson's decomposition~\cite{janson2004large}, and $\circled{2}$ is due to the sub-additivity of the supremum function (i.e., $\sup(a + b) \leq \sup(a) + \sup(b)$) and the linearity of the expectation.
    
    By defining $g(\tS) = \sum_{k=1}^K g_k(\tS_k)$ with each $g_k: \tS_k \mapsto \sum_{j \in [J_k]} \omega_{kj} g_{kj} (I_{kj})$ where each
    \begin{align*}
        g_{kj}: I_{kj} \mapsto \frac{1}{K m_k} \eE_{\tS_k^{\prime}} \left[ \sup_{\tf \in \wtF} \left( \sum_{i \in I_{kj}} \left( L(\tilde{\vx}_{ki}^{\prime}, \ty_{ki}^{\prime}, \tilde{f}_k) - L(\tilde{\vx}_{ki}, \ty_{ki}, \tilde{f}_k) \right) \right) \right]
    \end{align*}
    have differences bounded by $\frac{M}{K m_k}$ in the sense of the condition of Theorem~\ref{thm:new_mcdiarmid}; then for any $\delta \in (0,1)$, with probability at least $1 - \delta$, we have
    \begin{align*}
        \sup_{\tf \in \wtF} & \left(R(\tf) - \widehat{R}_{\tS} (\tf) \right) \\
        & \leq \underbrace{\frac{1}{K} \sum_{k=1}^K \sum_{j \in [J_k]} \frac{\omega_{kj}}{m_k} \eE_{\tS_k,\tS_k^{\prime}} \left[ \sup_{\tf \in \wtF} \left( \sum_{i \in I_{kj}} \left( L(\tilde{\vx}_{ki}^{\prime}, \ty_{ki}^{\prime}, \tilde{f}_k) - L(\tilde{\vx}_{ki}, \ty_{ki}, \tilde{f}_k) \right) \right) \right]}_{\overset{def}{=}\bigstar} + M \sqrt{\left( \frac{1}{K} \sum_{k=1}^K \frac{\chi_{f}(G_k)}{m_k} \right) \log\left( \frac{1}{\delta} \right)} .
    \end{align*}
    
    \textbf{Step 2: bound $\bigstar$ with respect to the fractional Rademacher complexity.}
    
    Next, taking the symmetrization technique by introduction of Rademacher variables, we have 
    \begin{align*}
        \bigstar & = \frac{1}{K} \sum_{k=1}^K \sum_{j \in [J_k]} \frac{\omega_{kj}}{m_k} \eE_{\tS_k,\tS_k^{\prime}} \left[ \sup_{\tf \in \wtF} \left( \sum_{i \in I_{kj}} \left( L(\tilde{\vx}_{ki}^{\prime}, \ty_{ki}^{\prime}, \tilde{f}_k) - L(\tilde{\vx}_{ki}, \ty_{ki}, \tilde{f}_k) \right) \right) \right] \\
        & = \frac{1}{K} \sum_{k=1}^K \sum_{j \in [J_k]} \frac{\omega_{kj}}{m_k} \eE_{\tS_k,\tS_k^{\prime}} \eE_{\boldsymbol{\sigma}} \left[ \sup_{\tf \in \wtF} \left( \sum_{i \in I_{kj}} \sigma_{ki} \left( L(\tilde{\vx}_{ki}^{\prime}, \ty_{ki}^{\prime}, \tilde{f}_k) - L(\tilde{\vx}_{ki}, \ty_{ki}, \tilde{f}_k) \right) \right) \right] \\
        & \overset{\circled{3}}{\leq} \frac{2}{K} \sum_{k=1}^K \eE_{\tS_k} \eE_{\boldsymbol{\sigma}} \left[ \frac{1}{m_k} \sum_{j \in [J_k]} \omega_{kj} \sup_{\tf \in \wtF} \left( \sum_{i \in I_{kj}} \sigma_{ki} L(\tilde{\vx}_{ki}, \ty_{ki}, \tilde{f}_k) \right) \right] \\
        & = 2 \mathfrak{R}_m^*(L \circ \wtF) .
    \end{align*}
    For a fixed pair $(k,i)$, $\sigma_{ki} = 1$ does not change anything but $\sigma_{ki} = -1$ consists in swapping both examples $(\tilde{\vx}_{ki}^{\prime}, \ty_{ki}^{\prime})$ and $(\tilde{\vx}_{ki}, \ty_{ki})$. Thus, when taking the expectations over $\tS_k$ and $\tS_k^{\prime}$, the introduction of Rademacher variables does not change the value.
    For $\circled{3}$, it is due to the sub-additivity of the supremum function and the linearity of the expectation.

    Thus, we can obtain 
    \begin{align*}
        \sup_{\tf \in \wtF} \left(R(\tf) - \widehat{R}_{\tS} (\tf) \right) \leq 2 \mathfrak{R}_m^*(L \circ \wtF) + M \sqrt{\left( \frac{1}{K} \sum_{k=1}^K \frac{\chi_{f}(G_k)}{m_k} \right) \log\left( \frac{1}{\delta} \right)} .
    \end{align*}
    Besides, based on the definition of supremum of functions, we have
    \begin{align*}
        \forall \tf \in \wtF, \ R(\tf) - \widehat{R}_{\tS} (\tf) \leq \sup_{\tf \in \wtF} \left(R(\tf) - \widehat{R}_{\tS} (\tf) \right). 
    \end{align*}
    Then, we can obtain the desired first bound~\eqref{eq:general_bound1}.
    
    \textbf{Step 3: bound the fractional Rademacher complexity with the empirical one.}
    
    By defining $g(\tS) = \sum_{k=1}^K g_k(\tS_k)$ with each $g_k: \tS_k \mapsto \sum_{j \in [J_k]} \omega_{kj} g_{kj} (I_{kj})$ where each
    \begin{align*}
        g_{kj}: I_{kj} \mapsto \frac{1}{K m_k} \eE_{\boldsymbol{\sigma}} \left[ \sum_{j \in [J_k]} \omega_{kj} \sup_{\tf \in \wtF} \left( \sum_{i \in I_{kj}} \sigma_{ki} L(\tilde{\vx}_{ki}, \ty_{ki}, \tilde{f}_k) \right) \right]
    \end{align*}
    having differences bounded by $\frac{M}{K m_k}$ in the sense of the condition of Theorem~\ref{thm:new_mcdiarmid}; then for any $\delta \in (0,1)$, with probability at least $1 - \delta$, we have
    \begin{align*}
        \mathfrak{R}_m^*(L \circ \wtF) \leq \widehat{\mathfrak{R}}_{\tS}^*(L \circ \wtF) + M \sqrt{\left( \frac{1}{K} \sum_{k=1}^K \frac{\chi_{f}(G_k)}{m_k} \right) \log\left( \frac{1}{\delta} \right)}.
    \end{align*}
    Then, we can get the desired second bound~\eqref{eq:general_bound2} by using the union bound with the first bound~\eqref{eq:general_bound1}.
\end{proof}

\section{Macro-AUC Maximization in MLC}
\label{sec:app_macro_auc_mlc}

\subsection{Proof of Theorem~\ref{thm:base_theorem_macroauc}}

\subsubsection{Problem transformation}
\label{sec:app_problem_transformation}

For the Macro-AUC maximization problem in multi-label learning, we can transform it into the problem of learning multiple tasks with graph-dependent examples which is considered in Section~\ref{sec:app_problem_setting}.

Specifically, construct the training dataset $\tS$ based on the original training set $S$ as follows. 
For each label $k \in [K]$, based on the original dataset $S_k$, construct the dataset $\tS_k = \{ (\tvx_{ki}, \ty_{ki}) \}_{i=1}^{m_k}$, where $\tvx_{ki} = (\tvx_{ki}^+, \tvx_{ki}^-)$, $\ty_{ki} = 1$,  and $(\tvx_{ki}^+, \tvx_{ki}^-) \in S_k^+ \times S_k^-$, $m_k = |S_k^+| |S_k^-| = n^2 \tau_k (1 - \tau_k)$ and let $\{ (I_{kj}, \omega_{kj})\}_{j \in [J_k]}$ be a fractional independent vertex cover of the dependence graph $G_k$ constructed over $\tS_k$ with $\sum_{j \in [J_k]} \omega_{kj} = \chi_{f} (G_k)$, where $\chi_{f} (G_k)$ is the fractional chromatic number of $G_k$. From previous results in bipartite ranking~\cite{usunier2005generalization,amini2015learning}, we know that 
\begin{align*}
    \forall k \in [K], \ \chi_{f} (G_k) = \max\{ |S_k^+|, |S_k^-|\} = (1 - \tau_k) n. 
\end{align*}
Besides, $\tf_k(\tvx_i) = f_k(\vx_i^+) - f_k(\vx_i^-)$ for each label $k \in [K]$.

\subsubsection{Proof of Theorem~\ref{thm:base_theorem_macroauc}}

\BaseTheoremMacroAUC*
\begin{proof}
    Based on the problem transformation in Section~\ref{sec:app_problem_transformation}, we can straightforwardly get this theorem by applying Theorem~\ref{thm:general_bound_multiple_tasks_dependent}.
\end{proof}


\subsection{Proof of Lemma~\ref{lem:relationship_losses} and~\ref{lem:relationship_risks}}

\subsubsection{Proof of Lemma~\ref{lem:relationship_losses}}

\RelationshipLosses*
\begin{proof}
    For the first inequality, the following holds:
    \begin{align*}
        L_{0/1}(\vx^+, \vx^-, f_k) = [\![ f_k(\vx^+) \leq f_k(\vx^-) ]\!]
        \leq \ell(f_k(\vx_p) - f_k(\vx_q))
        = L_{pa}(\vx^+, \vx^-, f_k) .
    \end{align*}
    For the second inequality, the following holds:
    \begin{align*}
        L_{0/1}(\vx^+, \vx^-, f_k) & = [\![ f_k(\vx^+) \leq f_k(\vx^-) ]\!] \\
        & \leq [\![ \sgn(f_k(\vx^+)) \leq \sgn(f_k(\vx^-)) ]\!] \\
        & = [\![ \sgn(f_k(\vx^+)) \neq +1 ]\!] + [\![ \sgn(f_k(\vx^-)) \neq -1 ]\!] - [\![ \sgn(f_k(\vx^+)) \neq +1 ]\!] [\![ \sgn(f_k(\vx^-)) \neq -1 ]\!] \\
        & \leq [\![ \sgn(f_k(\vx^+)) \neq +1 ]\!] + [\![ \sgn(f_k(\vx^-)) \neq -1 ]\!] \\
        & \leq \ell ( f_k(\vx^+)) +  \ell ( -f_k(\vx^-)) \\
        & = L_{u_2}(\vx^+, \vx^-, f_k) \\
        & = \frac{n}{\min\{|S_k^+|, |S_k^-|\}} \left( \frac{\min\{|S_k^+|, |S_k^-|\}}{n} \ell ( f_k(\vx^+)) + \frac{\min\{|S_k^+|, |S_k^-|\}}{n} \ell ( -f_k(\vx^-))  \right) \\
        & \leq \frac{1}{\tau_k} \left( \frac{|S_k^+|}{n} \ell ( f_k(\vx^+)) + \frac{|S_k^-|}{n} \ell ( -f_k(\vx^-))  \right) \\
        & = \frac{1}{\tau_k} L_{u_1}(\vx^+, \vx^-, f_k) \\
        & \leq \frac{1}{\tau_k} \left( \frac{\max\{|S_k^+|, |S_k^-|\}}{n} \ell ( f_k(\vx^+)) + \frac{\max\{|S_k^+|, |S_k^-|\}}{n} \ell ( -f_k(\vx^-))  \right) \\
        & \leq \frac{\max\{|S_k^+|, |S_k^-|\}}{\min\{|S_k^+|, |S_k^-|\}} \left( \ell ( f_k(\vx^+)) +  \ell ( -f_k(\vx^-)) \right) \\
        & = \frac{1 - \tau_k}{\tau_k} L_{u_2}(\vx^+, \vx^-, f_k) .
    \end{align*}
    Thus, the inequalities hold.
\end{proof}

\subsubsection{Proof of Lemma~\ref{lem:relationship_risks}}

\RelationshipRisks*
\begin{proof}
    For the first inequality, the following holds:
    \begin{align*}
        R_{0/1}(f) & = \frac{1}{K} \sum_{k=1}^K \eE_{\vx_p \sim P_k^+, \vx_q \sim P_k^-} [\![ f_k(\vx_p) \leq f_k(\vx_q) ]\!] \\
        & = \frac{1}{K} \sum_{k=1}^K \eE_{\vx_p \sim P_k^+, \vx_q \sim P_k^-} \left[ L_{0/1}(\vx_p, \vx_q, f_k) \right] \\
        & \leq \frac{1}{K} \sum_{k=1}^K \eE_{\vx_p \sim P_k^+, \vx_q \sim P_k^-} \left[ L_{pa}(\vx_p, \vx_q, f_k) \right] \\
        & = \frac{1}{K} \sum_{k=1}^K \eE_{\vx_p \sim P_k^+, \vx_q \sim P_k^-} \left[ \ell(f_k(\vx_p) - f_k(\vx_q)) \right] \\
        & = \eE_{S} \left[ \widehat{R}_{S}^{pa}(f) \right] \\
        & = R_{pa}(f)
    \end{align*}
    
    For the second inequality, we first have the following
    \begin{align*}
        R_{u_1} (f) & = \eE_{S} \left[ \widehat{R}_{S}^{u_1}(f) \right] \\
        & = \eE_{S} \left[ \frac{1}{K} \sum_{k=1}^K \frac{1}{|S_k^+| |S_k^-|} \sum_{(p, q) \in S_k^+ \times S_k^-} L_{u_1} (\vx_p, \vx_q, f_k) \right] \\
        & = \eE_{S} \left[ \frac{1}{K} \sum_{k=1}^K \frac{1}{|S_k^+| |S_k^-|} \sum_{(p, q) \in S_k^+ \times S_k^-} \left (\frac{|S_k^+|}{n} \ell ( f_k(\vx_p)) + \frac{|S_k^-|}{n} \ell ( -f_k(\vx_q)) \right) \right] \\
        & \geq \eE_{S} \left[ \frac{1}{K} \sum_{k=1}^K \frac{\tau_k}{|S_k^+| |S_k^-| } \sum_{(p, q) \in S_k^+ \times S_k^-} \left (\ell ( f_k(\vx_p)) + \ell ( -f_k(\vx_q)) \right) \right] .
    \end{align*} 
    Then, we can get
    \begin{align*}
        R_{0/1}(f) & = \frac{1}{K} \sum_{k=1}^K \eE_{\vx_p \sim P_k^+, \vx_q \sim P_k^-} \left[ L_{0/1}(\vx_p, \vx_q, f_k) \right] \\
        & \leq \frac{1}{K} \sum_{k=1}^K \eE_{\vx_p \sim P_k^+, \vx_q \sim P_k^-} \left[ L_{u_2}(\vx_p, \vx_q, f_k) \right] \\
        & = \frac{1}{K} \sum_{k=1}^K \eE_{\vx_p \sim P_k^+, \vx_q \sim P_k^-} \left[ \ell (f_k(\vx_p)) + \ell (-f_k(\vx_q)) \right] \\
        & = \eE_{S} \left[ \frac{1}{K} \sum_{k=1}^K \frac{1}{|S_k^+| |S_k^-|} \sum_{(p, q) \in S_k^+ \times S_k^-} \left( \ell (f_k(\vx_p)) + \ell (-f_k(\vx_q)) \right) \right] \\
        & = \eE_{S} \left[ \widehat{R}_{S}^{u_2}(f) \right] \\
        & = R_{u_2}(f) \\
        & \leq \eE_{S} \left[ \frac{1}{K} \sum_{k=1}^K \frac{1}{\tau_k |S_k^+| |S_k^-|} \sum_{(p, q) \in S_k^+ \times S_k^-} \left (\frac{|S_k^+|}{n} \ell ( f_k(\vx_p)) + \frac{|S_k^-|}{n} \ell ( -f_k(\vx_q)) \right) \right] \\
        & \leq \eE_{S} \left[ \frac{1}{K \tau_S^*} \sum_{k=1}^K \frac{1}{|S_k^+| |S_k^-|} \sum_{(p, q) \in S_k^+ \times S_k^-} \left (\frac{|S_k^+|}{n} \ell ( f_k(\vx_p)) + \frac{|S_k^-|}{n} \ell ( -f_k(\vx_q)) \right) \right] \\
        & = \eE_{S} \left[ \frac{1}{\tau_S^*} \widehat{R}_{S}^{u_1}(f) \right] \\
        & \leq \eE_{S} \left[ \frac{1}{K \tau_S^*} \sum_{k=1}^K \frac{\max\{S_k^+, S_k^-\}}{n |S_k^+| |S_k^-|} \sum_{(p, q) \in S_k^+ \times S_k^-} \left ( \ell ( f_k(\vx_p)) + \ell ( -f_k(\vx_q)) \right) \right] \\
        & \leq \eE_{S} \left[ \frac{1}{K \tau_S^*} \sum_{k=1}^K \frac{1 - \tau_k}{|S_k^+| |S_k^-|} \sum_{(p, q) \in S_k^+ \times S_k^-} \left ( \ell ( f_k(\vx_p)) + \ell ( -f_k(\vx_q)) \right) \right] \\
        & \leq \eE_{S} \left[ \frac{1 - \tau_S^*}{K \tau_S^*} \sum_{k=1}^K \frac{1 - \tau_k}{|S_k^+| |S_k^-|} \sum_{(p, q) \in S_k^+ \times S_k^-} \left ( \ell ( f_k(\vx_p)) + \ell ( -f_k(\vx_q)) \right) \right] \\
        & = \eE_{S} \left[ \frac{1 - \tau_S^*}{\tau_S^*} \widehat{R}_{S}^{u_2}(f) \right] .
    \end{align*}
    Thus, the second inequality holds.
\end{proof}

\subsection{Proof of Theorem~\ref{thm:learning_guarantee_pa}, Corollary~\ref{cor:learning_guarantee_pa_balanced} and~\ref{cor:learning_guarantee_pa_extreme_imbalanced}}

\subsubsection{The fractional Rademacher complexity of the hypothesis space}

\begin{definition}[\textbf{The fractional Rademacher complexity of the hypothesis space for $L_{pa}$}]
    Given a dataset $S$ (and its corresponding constructed dataset $\tS$), 
    define the empirical fractional Rademacher complexity of the hypothesis space $\wtF$ w.r.t. $\tS$ as follows:
    \begin{align*}
        \widehat{\mathfrak{R}}_{\tS}^*(\wtF) & = \frac{1}{K} \sum_{k=1}^K \eE_{\boldsymbol{\sigma}} \left[ \frac{1}{m_k} \sum_{j \in [J_k]} \omega_{kj} \sup_{\tf \in \wtF} \left( \sum_{i \in I_{kj}} \sigma_{ki} \tf_k(\tvx_{ki}) \right) \right] \\
        & = \frac{1}{K} \sum_{k=1}^K \eE_{\boldsymbol{\sigma}} \left[ \frac{1}{m_k} \sum_{j \in [J_k]} \omega_{kj} \sup_{\tf \in \wtF} \left( \sum_{i \in I_{kj}} \sigma_{ki} (f_k(\tvx_{ki}^+ - f_k(\tvx_{ki}^-)) \right) \right] .
    \end{align*}
\end{definition}

\begin{lemma}[\textbf{The fractional Rademacher complexity of the kernel-based hypothesis space for $L_{pa}$}]
\label{lem:app_frac_rade_kernel_hypothesis_pairwise}
Suppose (1) and (2) in {\rm Assumption~\ref{assump_common}} hold. Then, for the kernel-based hypothesis space~\eqref{eq:kernel_hypothesis}, its empirical fractional Rademacher complexity w.r.t. the dataset $\tS$, can be bounded as bellow:
    \begin{align*}
        \widehat{\mathfrak{R}}_{\tS}^*(\wtF) \leq \frac{2 B r}{\sqrt{n}} \left( \frac{1}{K} \sum_{k=1}^K \sqrt{\frac{1}{\tau_k}} \right)
    \end{align*}

\end{lemma}

\begin{proof}
    By the definition of $\widehat{\mathfrak{R}}_{\tS}^*(\wtF)$, we have
    \begin{align*}
        \widehat{\mathfrak{R}}_{\tS}^*(\wtF) & = \frac{1}{K} \sum_{k=1}^K \eE_{\boldsymbol{\sigma}} \left[ \frac{1}{m_k} \sum_{j \in [J_k]} \omega_{kj} \sup_{\tf \in \wtF} \left( \sum_{i \in I_{kj}} \sigma_{ki} \tf_k(\tvx_{ki}) \right) \right]  \\
        & = \frac{1}{K} \sum_{k=1}^K \eE_{\boldsymbol{\sigma}} \left[ \frac{1}{m_k} \sum_{j \in [J_k]} \omega_{kj} \sup_{\| \vw_k \| \leq B} \left( \sum_{i \in I_{kj}} \sigma_{ki} \langle \vw_{k}, \Phi(\tvx_{ki}) \rangle \right) \right]  \\
        & \leq \frac{1}{K} \sum_{k=1}^K \eE_{\boldsymbol{\sigma}} \left[ \frac{1}{m_k} \sum_{j \in [J_k]} \omega_{kj} \| \sup_{\| \vw_k \| \leq B} \| \vw_{k} \| \left \| \sum_{i \in I_{kj}} \sigma_{ki} \Phi(\tvx_{ki}) \right \|  \right] \qquad (\text{Cauchy–Schwarz inequality}) \\
        & = \frac{1}{K} \sum_{k=1}^K \eE_{\boldsymbol{\sigma}} \left[ \frac{B}{m_k} \sum_{j \in [J_k]} \omega_{kj}  \left \| \sum_{i \in I_{kj}} \sigma_{ki} \Phi(\tvx_{ki}) \right \|  \right] \qquad (\text{the definition of sup}) \\
        & = \frac{1}{K} \sum_{k=1}^K \frac{B}{m_k} \sum_{j \in [J_k]} \omega_{kj} \eE_{\boldsymbol{\sigma}} \left[ \left \| \sum_{i \in I_{kj}} \sigma_{ki} \Phi(\tvx_{ki}) \right \|  \right] \qquad (\text{linearity of expectation}) \\
        & \leq \frac{1}{K} \sum_{k=1}^K \frac{B}{m_k} \sum_{j \in [J_k]} \omega_{kj}  \left ( \eE_{\boldsymbol{\sigma}} \left[ \left \| \sum_{i \in I_{kj}} \sigma_{ki} \Phi(\tvx_{ki}) \right \|^2 \right] \right )^{\frac{1}{2}} \qquad (\text{Jensen's inequality}) \\
        & = \frac{1}{K} \sum_{k=1}^K \frac{B}{m_k} \sum_{j \in [J_k]} \omega_{kj}  \left ( \eE_{\boldsymbol{\sigma}} \left[ \sum_{p \in I_{kj}, q \in I_{kj}} \sigma_{kp} \sigma_{kq} \left \langle \Phi(\tvx_{kp}), \Phi(\tvx_{kq}) \right \rangle \right] \right )^{\frac{1}{2}} \qquad  \\
        & = \frac{1}{K} \sum_{k=1}^K \frac{B}{m_k} \sum_{j \in [J_k]} \omega_{kj}  \left ( \sum_{i \in I_{kj}} \left \langle \Phi(\tvx_{ki}), \Phi(\tvx_{ki}) \right \rangle \right )^{\frac{1}{2}} \qquad (\forall p \neq q, \eE[\sigma_{kp}\sigma_{kq}] = 0 \text{~and~} \eE[\sigma_{ki}\sigma_{ki}] = 1) \\
        & = \frac{1}{K} \sum_{k=1}^K \frac{B}{m_k} \sum_{j \in [J_k]} \omega_{kj}  \left ( \sum_{i \in I_{kj}} \left \langle \Phi(\vx_{ki}^+) - \Phi(\vx_{ki}^-), \Phi(\vx_{ki}^+) - \Phi(\vx_{ki}^-) \right \rangle \right )^{\frac{1}{2}}  \qquad  (\Phi(\tvx) = \Phi(\vx^+) - \Phi(\vx^-)) \\
        & \leq \frac{1}{K} \sum_{k=1}^K \frac{B}{m_k} \sum_{j \in [J_k]} \omega_{kj}  \left ( \sum_{i \in I_{kj}} \left( \left \langle \Phi(\vx_{ki}^+), \Phi(\vx_{ki}^+) \right \rangle + \left \langle \Phi(\vx_{ki}^-), \Phi(\vx_{ki}^-) \right \rangle \right) \right )^{\frac{1}{2}}  \qquad (\| a + b \|^2 \leq 2 \| a \|^2 + 2 \| b \|^2) \\
        & = \frac{1}{K} \sum_{k=1}^K \frac{B}{m_k} \sum_{j \in [J_k]} \omega_{kj}  \left ( \sum_{i \in I_{kj}} 2 \left( \kappa(\vx_{ki}^+, \vx_{ki}^+) + \kappa(\vx_{ki}^-, \vx_{ki}^-) \right) \right )^{\frac{1}{2}} \qquad (\kappa (\vx, \vx) = \left \langle \Phi(\vx), \Phi(\vx) \right \rangle)   \\
        & \leq \frac{1}{K} \sum_{k=1}^K \frac{2 B r}{m_k} \sum_{j \in [J_k]} \omega_{kj} \sqrt{m_{kj}} \qquad (\kappa (\vx, \vx) \leq r^2 \text{~and let~} |I_{kj}| = m_{kj}) \\
        & = \frac{1}{K} \sum_{k=1}^K \frac{2 B r \chi_{f}(G_k)}{m_k} \sum_{j \in [J_k]} \frac{\omega_{kj}}{\chi_{f}(G_k)} \sqrt{m_{kj}} \qquad  \\
        & \leq \frac{1}{K} \sum_{k=1}^K \frac{2 B r \sqrt{\chi_{f}(G_k)}}{m_k} \sqrt{\sum_{j \in [J_k]} \omega_{kj} m_{kj}} \qquad  (\sum_{j \in [J_k]} \frac{\omega_{kj}}{\chi_{f}(G_k)} = 1 \text{~and Jensen's inequality}) \\
        & = \frac{1}{K} \sum_{k=1}^K 2 B r \sqrt{\frac{\chi_{f}(G_k)}{m_k}}  \qquad (\sum_{j \in [J_k]} \omega_{kj} m_{kj} = m_k)
    \end{align*}
    Since for Macro-AUC optimization in multi-label learning, $\chi_{f}(G_k) = \max\{|S_k^+|, |S_k^-|\}$, $m_k = |S_k^+| |S_k^-|$ and $\tau_k = \min\{|S_k^+|, |S_k^-|\}$ hold, then we can get 
    \begin{align*}
        \frac{1}{K} \sum_{k=1}^K 2 B r \sqrt{\frac{\chi_{f}(G_k)}{m_k}} = \frac{2 B r}{\sqrt{n}} \left( \frac{1}{K} \sum_{k=1}^K \sqrt{\frac{1}{\tau_k}} \right) .
    \end{align*}
    Thus, we can get the desired result by combining the above equality and previous inequality.
\end{proof}

\subsubsection{The contraction inequality}

\begin{lemma}[\textbf{Contraction inequality for $L_{pa}$}]
\label{lem:app_contraction_pairwise}
    Assume the loss function 
    $L_{\phi} = L_1(y, \tf_k(\tvx))$
    is $\mu$-Lipschitz continuous w.r.t. the second argument where $L_1$ denotes the loss function of two inputs for convenience. Then, the following holds:
    \begin{align*}
        \widehat{\mathfrak{R}}_{\tS}^*(L_{\phi} \circ \mathcal{F}) \leq \mu \widehat{\mathfrak{R}}^*(\wtF) .
    \end{align*}
\end{lemma}
\begin{proof}
    Since $\widehat{\mathfrak{R}}_{\tS}^*(L_{\phi} \circ \mathcal{F}) = \frac{1}{K} \sum_{k=1}^K \widehat{\mathfrak{R}}_{\tS_k}^*(L_{\phi} \circ \mathcal{F}_k)$ and $\widehat{\mathfrak{R}}_{\tS}^*( \mathcal{F}) = \frac{1}{K} \sum_{k=1}^K \widehat{\mathfrak{R}}_{\tS_k}^*(\mathcal{F}_k)$, we first prove $\widehat{\mathfrak{R}}_{\tS_k}^*(L_{\phi} \circ \mathcal{F}_k) \leq \mu \widehat{\mathfrak{R}}_{\tS_k}^*(\mathcal{F}_k)$ and then can get the desired result.
    
    Here we prove the inequality $\widehat{\mathfrak{R}}_{\tS_k}^*(L_{\phi} \circ \mathcal{F}_k) \leq \mu \widehat{\mathfrak{R}}_{\tS_k}^*(\mathcal{F}_k)$ following the idea in~\citet{mohri2018foundations} (Lemma~5.7, p.93), and we omit the index $k$ and the symbol $\tS_k$ for notation simplicity in the following.
    
    First we fix a sample $(\tvx_1,\dots,\tvx_{m})$, then by defintion,
    \begin{align*}
        \widehat{\mathfrak{R}}^*(L_{\phi} \circ f) & =  \eE_{\boldsymbol{\sigma}} \left[ \frac{1}{m} \sum_{j \in [J]} w_{j} \sup_{\tf \in \wtF} \left( \sum_{i \in I_{j}} \sigma_{i} L_1(y_i, \tf(\tvx_i)) \right) \right] \\
        & = \frac{1}{m} \sum_{j \in [J]} w_{j} \eE_{\sigma_1,\dots,\sigma_{n_j-1}} \left[ \eE_{\sigma_{n_j}} \left[ \sup_{\tf \in \wtF} u_{n_j - 1} (\tf) + \sigma_{n_j} L_1(y_{n_j}, \tf(\tvx_{n_j}))  \right ]\right] , \qquad (\text{denote~} n_j = |I_j| \text{~for simplicity})
    \end{align*}
    where $u_{n_j- 1} (\tf) = \sum_{i=1}^{n_j} \sigma_i L_1(y_i, \tf(\tvx_i))$. By the definition of the supremum, for any $\epsilon > 0$, there exists $\tf^1, \tf^2 \in \wtF$ such that
    \begin{align*}
        u_{n_j - 1} (\tf^1) +  L_1(y_{n_j}, \tf^1(\tvx_{n_j})) \leq (1 - \epsilon) \left[ \sup_{\tf \in \wtF} u_{n_j - 1} (\tf) +  L_1(y_{n_j}, \tf(\tvx_{n_j})) \right] ,
    \end{align*}
    and 
    \begin{align*}
        u_{n_j - 1} (\tf^2) -  L_1(y_{n_j}, \tf^2(\tvx_{n_j})) \leq (1 - \epsilon) \left[ \sup_{\tf \in \wtF} u_{n_j - 1} (\tf) -  L_1(y_{n_j}, \tf(\tvx_{n_j})) \right] .
    \end{align*}
    Thus, for any $\epsilon > 0$, by definition of $\eE_{\sigma_{n_j}}$,
    \begin{align*}
        & (1 - \epsilon) \eE_{\sigma_{n_j}} \left[ \sup_{\tf \in \wtF} u_{n_j - 1} (\tf) + \sigma_{n_j} L_1(y_{n_j}, \tf(\tvx_{n_j}))  \right ] \\
        & = (1 - \epsilon) \left[ \frac{1}{2} \sup_{\tf \in \wtF} u_{n_j - 1} (\tf) +  L_1(y_{n_j}, \tf(\tvx_{n_j}))  \right ] + \left[ \frac{1}{2} \sup_{\tf \in \wtF} u_{n_j - 1} (\tf) -  L_1(y_{n_j}, \tf(\tvx_{n_j}))  \right ] \\
        & \leq \frac{1}{2} \left[ u_{n_j - 1} (\tf^1) +  L_1(y_{n_j}, \tf^1(\tvx_{n_j})) \right] + \frac{1}{2} \left[ u_{n_j - 1} (\tf^2) -  L_1(y_{n_j}, \tf^2(\tvx_{n_j})) \right].
    \end{align*}
    Let $s = \sgn(\tf^1(\tvx_{n_j}) - \tf^2(\tvx_{n_j}))$. Then, the previous inequality implies
    \begin{align*}
        & (1 - \epsilon) \eE_{\sigma_{n_j}} \left[ \sup_{\tf \in \wtF} u_{n_j - 1} (\tf) + \sigma_{n_j} L_1(y_{n_j}, \tf(\tvx_{n_j}))  \right ] &~ \\
        & \leq \frac{1}{2} \left[ u_{n_j - 1} (\tf^1) + u_{n_j - 1} (\tf^2) + s \mu (\tf^1(\tvx_{n_j}) - \tf^2(\tvx_{n_j})) \right] & (\text{Lipschitz property}) \\
        & = \frac{1}{2} \left[ u_{n_j - 1} (\tf^1) + s \mu \tf^1(\tvx_{n_j}) \right] + \frac{1}{2} \left[ u_{n_j - 1} (\tf^2) - s \mu \tf^2(\tvx_{n_j}) \right] & (\text{rearranging}) \\
        & \leq \frac{1}{2} \sup_{\tf \in \wtF} \left[ u_{n_j - 1} (\tf) + s \mu \tf(\tvx_{n_j}) \right] + \frac{1}{2} \sup_{\tf \in \wtF} \left[ u_{n_j - 1} (\tf) - s \mu \tf(\tvx_{n_j}) \right] & (\text{definition of sup}) \\
        & = \eE_{\sigma_{n_j}} \left[ u_{n_j - 1} (\tf) + \sigma_{n_j} \mu \tf(\tvx_{n_j}) \right] . & (\text{definition of~}\eE_{\sigma_{n_j}}) \\
    \end{align*}
    Since the inequality holds for any $\epsilon > 0$, we have
    \begin{align*}
        \eE_{\sigma_{n_j}} \left[ \sup_{\tf \in \wtF} u_{n_j - 1} (\tf) + \sigma_{n_j} L_1(y_{n_j}, \tf(\tvx_{n_j}))  \right ] \leq \eE_{\sigma_{n_j}} \left[ u_{n_j - 1} (\tf) + \sigma_{n_j} \mu \tf(\tvx_{n_j}) \right].
    \end{align*}
    Proceeding in the same way for all other $\sigma_i$ ($i \in [I_j], i \neq n_j$) proves that 
    \begin{align*}
         \eE_{\boldsymbol{\sigma}} \left[ \sup_{\tf \in \wtF} \left( \sum_{i \in I_{j}} \sigma_{i} L_1(y_i, \tf(\tvx_i)) \right) \right] \leq  \eE_{\boldsymbol{\sigma}} \left[ \sup_{\tf \in \wtF} \left( \sum_{i \in I_{j}} \sigma_{i} \mu \tf(\tvx_i) \right) \right]. 
    \end{align*}
    By proceeding other $j \in [J]$, we can obtain the following 
    \begin{align*}
        \widehat{\mathfrak{R}}^*(L_{\phi} \circ f) = \frac{1}{m} \sum_{j \in [J]} w_{j} \eE_{\boldsymbol{\sigma}} \left[ \sup_{\tf \in \wtF} \left( \sum_{i \in I_{j}} \sigma_{i} L_1(y_i, \tf(\tvx_i)) \right) \right] \leq \frac{1}{m} \sum_{j \in [J]} w_{j} \eE_{\boldsymbol{\sigma}} \left[ \sup_{\tf \in \wtF} \left( \sum_{i \in I_{j}} \sigma_{i} \mu \tf(\tvx_i) \right) \right] = \mu \widehat{\mathfrak{R}}^*(\tf) .
    \end{align*}
\end{proof}

\subsubsection{Proof of Theorem~\ref{thm:learning_guarantee_pa}}

\LearningGuaranteePa*
\begin{proof}
    Since the base loss $\ell(z)$ is bounded by $B$ and the loss $L_{pa}(\vx^+, \vx^-, f_k) = \ell \left( f_k(\vx^+) - f_k(\vx^-) \right)$, the loss $L_{\phi} = L_{pa}$ is bounded by $B$. Then, applying Theorem~\ref{thm:base_theorem_macroauc}, and combining Lemma~\ref{lem:app_contraction_pairwise} and Lemma~\ref{lem:app_frac_rade_kernel_hypothesis_pairwise}, we can obtain that for any $\delta > 0$, the following generalization bound holds with probability at least $1 - \delta$ over the draw of an i.i.d. sample $S$ of size $n$:
    \begin{align*}
        R_{pa} (f) \leq \widehat{R}_{pa} (f) + \frac{4 \rho r \Lambda}{\sqrt{n}} \left (\frac{1}{K} \sum_{k=1}^K \sqrt{\frac{1}{\tau_k}} \right ) + 3 B \sqrt{\frac{\log(\frac{2}{\delta})}{2n}} \left ( \sqrt{\frac{1}{K} \sum_{k=1}^K \frac{1}{\tau_k}} \right ).
    \end{align*}
    Finally, we can get the desired result by combining the above inequality and Lemma~\ref{lem:relationship_risks} (i.e., $R_{0/1} (f) \leq R_{pa} (f)$).
\end{proof}

\subsubsection{Proof of Corollary~\ref{cor:learning_guarantee_pa_balanced}}
\label{sec:app_learning_guarantee_pa_balanced}
\begin{restatable}[\textbf{Learning guarantee of $\mathcal{A}^{pa}$ in balanced case}]
    {corollary}{LearningGuaranteePaBalanced}
\label{cor:learning_guarantee_pa_balanced}
    Assume the loss $L_{\phi} = L_{pa}$, where $L_{pa}$ is defined in Eq.\eqref{eq:surrogate_pa}. Besides, {\rm Assumption~\ref{assump_common}} holds and suppose $S$ is balanced. Then, for any $\delta > 0$, with probability at least $1 - \delta$ over the draw of an i.i.d. sample $S$ of size $n$, the following generalization bound holds for any $f \in \mathcal{F}$: 
    \begin{align}
        R_{0/1} (f) \leq R_{pa} (f) \leq & \widehat{R}_{pa} (f) + \frac{4 \sqrt{2} \rho r \Lambda}{\sqrt{n}} \ + \nonumber\\
        & 3 \sqrt{2} B \sqrt{\frac{\log(\frac{2}{\delta})}{2n}} .
    \end{align}
\end{restatable}
\begin{proof}
    It is straightforward to get the result by applying the Theorem~\ref{thm:learning_guarantee_pa} by plugging $\tau_k = \frac{1}{2}$.
\end{proof}

\subsubsection{Proof of Corollary~\ref{cor:learning_guarantee_pa_extreme_imbalanced}}
\label{sec:app_learning_guarantee_pa_extreme_imbalanced}

\LearningGuaranteePaExtremeImbalanced*
\begin{proof}
    It is straightforward to get the result by applying the Theorem~\ref{thm:learning_guarantee_pa} by plugging $\tau_k = \frac{1}{n}$.
\end{proof}

\subsection{Proof of Theorem~\ref{thm:learning_guarantee_u1} and~\ref{thm:learning_guarantee_u2}, Corollary~\ref{cor:learning_guarantee_u_balanced},~\ref{cor:learning_guarantee_u1_extreme_imbalanced} and~\ref{cor:learning_guarantee_u2_extreme_imbalanced}}

\subsubsection{The fractional Rademacher complexity of the hypothesis space}

\begin{definition}[\textbf{The fractional Rademacher complexity of the hypothesis space w.r.t. (reweighted) univariate losses}]
    Given a dataset $S$ (and its corresponding constructed dataset $\tS$), assume the loss function $L_{\phi} = L_{u}(\vx^+, \vx^-, f_k) = a^+ (S_k) \ell(f_k(\vx^+)) + a^- (S_k) \ell(-f_k(\vx^-))$ for each label $k \in [K]$, where $\ell(t)$ is the base loss function and the reweighting function $a^+ (S_k)$ (or $a^- (S_k)$) indicates its dependency on $S_k$. Then, define the empirical fractional Rademacher complexity of the hypothesis space $\mathcal{F}$ w.r.t. $\tS$ and $L_{u}$ as follows:
    \begin{align*}
        \widehat{\mathfrak{R}}_{\tS, L_u}^*(\mathcal{F}) = \frac{1}{K} \sum_{k=1}^K \eE_{\boldsymbol{\sigma}} \left[ \frac{1}{m_k} \sum_{j \in [J_k]} \omega_{kj} \sup_{f \in \mathcal{F}} \left( \sum_{i \in I_{kj}} (\sigma_{ki}^+ a^+ (S_k) f_k(\vx_{ki}^+) + \sigma_{ki}^- a^- (S_k) f_k(\vx_{ki}^-) )  \right) \right].
    \end{align*}
\end{definition}

\begin{lemma}[\textbf{The fractional Rademacher complexity of kernel-based hypothesis space w.r.t. (reweighted) univariate losses}]
\label{lem:app_frac_rade_kernel_hypothesis_univariate_general}
Suppose (1) and (2) in {\rm Assumption~\ref{assump_common}} hold and the loss function $L_{u}(\vx^+, \vx^-, f_k) = a^+ (S_k) \ell(f_k(\vx^+)) + a^- (S_k) \ell(-f_k(\vx^-))$. Then, for the kernel-based hypothesis space~\eqref{eq:kernel_hypothesis}, its empirical fractional Rademacher complexity w.r.t. the dataset $\tS$ and loss function $L_u$, can be bounded as bellow:

    \begin{align*}
        \widehat{\mathfrak{R}}_{\tS, L_u}^*(\mathcal{F}) \leq \frac{B r}{\sqrt{n}} \left( \frac{1}{K} \sum_{k=1}^K \sqrt{\frac{1}{\tau_k}} \left( a^+(S_k) + a^-(S_k) \right) \right) .
    \end{align*}
    
\end{lemma}

\begin{proof}
    By the definition of $\widehat{\mathfrak{R}}_{\tS, L_u}^*(\mathcal{F})$ and let $a_k^+$ (or $a_k^-$) denote $a^+(S_k)$ (or $a^-(S_k)$) for notation simplicity, we can have
    \begin{align*}
        \widehat{\mathfrak{R}}_{\tS, L_u}^*(\mathcal{F}) & = \frac{1}{K} \sum_{k=1}^K \eE_{\boldsymbol{\sigma}} \left[ \frac{1}{m_k} \sum_{j \in [J_k]} \omega_{kj} \sup_{f \in \mathcal{F}} \left( \sum_{i \in I_{kj}} \left (\sigma_{ki}^+ a^+ (S_k) f_k(\vx_{ki}^+) + \sigma_{ki}^- a^- (S_k) f_k(\vx_{ki}^-) \right)  \right) \right]  \\
        & = \frac{1}{K} \sum_{k=1}^K \eE_{\boldsymbol{\sigma}} \left[ \frac{1}{m_k} \sum_{j \in [J_k]} \omega_{kj} \sup_{\|\vw_k\| \leq B} \left( \sum_{i \in I_{kj}} (\sigma_{ki}^+ a_k^+ \langle \vw_k, \Phi(\vx_{ki}^+) \rangle + \sigma_{ki}^- a_k^- \langle \vw_k, \Phi(\vx_{ki}^-) \rangle )  \right) \right]  \\
        & \overset{\circled{1}}{\leq} \frac{1}{K} \sum_{k=1}^K \eE_{\boldsymbol{\sigma}} \left[ \frac{1}{m_k} \sum_{j \in [J_k]} \omega_{kj} \left ( \sup_{\|\vw_k\| \leq B} \left( \sum_{i \in I_{kj}} \sigma_{ki}^+ a_k^+ \langle \vw_k, \Phi(\vx_{ki}^+) \rangle \right) + \sup_{\|\vw_k\| \leq B} \left( \sum_{i \in I_{kj}} \sigma_{ki}^- a_k^- \langle \vw_k, \Phi(\vx_{ki}^-) \rangle   \right) \right) \right] \\
        & \overset{def}{=} \clubsuit + \spadesuit
    \end{align*}
    where the inequality $\circled{1}$ is due to sub-additivity of the supremum function (i.e., $\sup (a + b) \leq \sup a + \sup b$), and
    \begin{align*}
        \clubsuit & = \frac{1}{K} \sum_{k=1}^K \eE_{\boldsymbol{\sigma}} \left[ \frac{1}{m_k} \sum_{j \in [J_k]} \omega_{kj} \sup_{\|\vw_k\| \leq B} \left( \sum_{i \in I_{kj}} \sigma_{ki}^+ a_k^+ \langle \vw_k, \Phi(\vx_{ki}^+) \rangle \right) \right] , \\
        \spadesuit & = \frac{1}{K} \sum_{k=1}^K \eE_{\boldsymbol{\sigma}} \left[ \frac{1}{m_k} \sum_{j \in [J_k]} \omega_{kj} \sup_{\|\vw_k\| \leq B} \left( \sum_{i \in I_{kj}} \sigma_{ki}^- a_k^- \langle \vw_k, \Phi(\vx_{ki}^-) \rangle \right) \right] .
    \end{align*}
    Next, we will (upper) bound $\clubsuit$ and $\spadesuit$, respectively.
    
    Firstly, for $\clubsuit$, we can have
    \begin{align*}
        \clubsuit & = \frac{1}{K} \sum_{k=1}^K \eE_{\boldsymbol{\sigma}} \left[ \frac{1}{m_k} \sum_{j \in [J_k]} \omega_{kj} \sup_{\|\vw_k\| \leq B} \left( \sum_{i \in I_{kj}} \sigma_{ki}^+ a_k^+ \langle \vw_k, \Phi(\vx_{ki}^+) \rangle \right) \right] \\
        & \leq \frac{1}{K} \sum_{k=1}^K \eE_{\boldsymbol{\sigma}} \left[ \frac{a_k^+}{m_k} \sum_{j \in [J_k]} \omega_{kj} \sup_{\|\vw_k\| \leq B} \| \vw_k \| \left\| \sum_{i \in I_{kj}} \sigma_{ki}^+  \Phi(\vx_{ki}^+) \right\| \right] \qquad (\text{Cauchy–Schwarz inequality}) \\
        & = \frac{1}{K} \sum_{k=1}^K \frac{B a_k^+}{m_k} \sum_{j \in [J_k]} \omega_{kj} \eE_{\boldsymbol{\sigma}} \left[ \left\| \sum_{i \in I_{kj}} \sigma_{ki}^+ \Phi(\vx_{ki}^+) \right\| \right] \qquad (\text{the definition of the sup and linearity of expectation}) \\
        & \leq \frac{1}{K} \sum_{k=1}^K \frac{B a_k^+}{m_k} \sum_{j \in [J_k]} \omega_{kj} \left( \eE_{\boldsymbol{\sigma}} \left[ \left\| \sum_{i \in I_{kj}} \sigma_{ki}^+ \Phi(\vx_{ki}^+) \right\|^2 \right] \right)^{\frac{1}{2}} \qquad (\text{Jensen's inequality}) \\
        & = \frac{1}{K} \sum_{k=1}^K \frac{B a_k^+}{m_k} \sum_{j \in [J_k]} \omega_{kj} \left( \eE_{\boldsymbol{\sigma}} \left[ \sum_{p \in I_{kj}, q \in I_{kj}} \sigma_{kp}^+ \sigma_{kq}^+ \langle \Phi(\vx_{kp}^+), \Phi(\vx_{kq}^+) \rangle \right] \right)^{\frac{1}{2}} \\
        & = \frac{1}{K} \sum_{k=1}^K \frac{B a_k^+}{m_k} \sum_{j \in [J_k]} \omega_{kj} \left( \sum_{i \in I_{kj}} \langle \Phi(\vx_{ki}^+), \Phi(\vx_{ki}^+) \rangle \right)^{\frac{1}{2}} \quad(\forall p \neq q, \eE[\sigma_{kp}\sigma_{kq}] = 0 \text{~and~} \eE[\sigma_{ki}\sigma_{ki}] = 1) \\
        & \leq \frac{1}{K} \sum_{k=1}^K \frac{B r a_k^+}{m_k} \sum_{j \in [J_k]} \omega_{kj} \sqrt{m_{kj}}
        \qquad \qquad (\langle \Phi(\vx_{ki}^+), \Phi(\vx_{ki}^+) \rangle = \kappa(\vx_{ki}^+, \vx_{ki}^+) \leq r^2 \text{~and let~} m_{kj} = |I_{kj}|) \\
        & = \frac{1}{K} \sum_{k=1}^K \frac{B r a_k^+ \chi_{f} (G_k)}{m_k} \sum_{j \in [J_k]} \frac{\omega_{kj}}{\chi_{f} (G_k)} \sqrt{m_{kj}} \\
        & \leq \frac{1}{K} \sum_{k=1}^K \frac{B r a_k^+ \sqrt{\chi_{f} (G_k)}}{m_k} \sqrt{\sum_{j \in [J_k]} \omega_{kj} m_{kj}} \qquad \qquad (\sum_{j \in [J_k]} \frac{\omega_{kj}}{\chi_{f} (G_k)} = 1 \text{~and Jensen's inequality}) .
    \end{align*}
    Since $\sum_{j \in [J_k]} \omega_{kj} m_{kj} = m_k , \chi_{f}(G_k) = \max\{|S_k^+|, |S_k^-|\}, m_k = |S_k^+| |S_k^-|$ and $\tau_k = \min\{|S_k^+|, |S_k^-|\}$ hold, it comes
    \begin{align*}
        \clubsuit \leq \frac{1}{K} \sum_{k=1}^K \frac{B r a_k^+ \sqrt{\chi_{f} (G_k)}}{m_k} \sqrt{\sum_{j \in [J_k]} \omega_{kj} m_{kj}} \leq \frac{B r}{\sqrt{n}} \left( \frac{1}{K} \sum_{k=1}^K a_k^+ \sqrt{\frac{1}{\tau_k}} \right) .
    \end{align*}
    Similarly to the proof of $\clubsuit$, we can obtain the upper bound for $\spadesuit$:
    \begin{align*}
        \spadesuit \leq \frac{B r}{\sqrt{n}} \left( \frac{1}{K} \sum_{k=1}^K a_k^- \sqrt{\frac{1}{\tau_k}} \right).
    \end{align*}
    Thus, we can obtain the final result:
    \begin{align*}
        \widehat{\mathfrak{R}}_{\tS, L_u}^*(\mathcal{F}) \leq \frac{B r}{\sqrt{n}} \left( \frac{1}{K} \sum_{k=1}^K \sqrt{\frac{1}{\tau_k}} \left( a^+(S_k) + a^-(S_k) \right) \right) .
    \end{align*}
\end{proof}

\begin{proposition} [The fractional Rademacher complexity of the kernel-based hypothesis space w.r.t. $L_{u_1}$ and $L_{u_2}$]
\label{pro:app_frac_rade_kernel_hypothesis_univariate_specific}
For the surrogate loss functions $L_{u_1}$ and $L_{u_2}$, which are defined in Eq.\eqref{eq:surrogate_u1} and Eq.\eqref{eq:surrogate_u2} respectively, we have
    \begin{align*}
        \widehat{\mathfrak{R}}_{\tS, L_{u_1}}^*(\mathcal{F}) & \leq \frac{B r}{\sqrt{n}} \left( \frac{1}{K} \sum_{k=1}^K \sqrt{\frac{1}{\tau_k}} \right) , \\
        \widehat{\mathfrak{R}}_{\tS, L_{u_2}}^*(\mathcal{F}) & \leq \frac{2 B r}{\sqrt{n}} \left( \frac{1}{K} \sum_{k=1}^K \sqrt{\frac{1}{\tau_k}} \right) .
    \end{align*}
\end{proposition}
\begin{proof}
    The proof is straightforward based on Lemma~\ref{lem:app_frac_rade_kernel_hypothesis_univariate_general} by plugging in the specific reweighted values (i.e., $a^+(S_k)$ and $a^-(S_k)$) of the surrogate univariate loss functions.
\end{proof}

\subsubsection{The contraction inequality}

\begin{lemma}[\textbf{Contraction inequality for the (reweighted) univariate loss $L_u$}]
\label{lem:app_contraction_univariate}
    For a dataset $S$ (and its corresponding constructed dataset $\tS$), assume the loss function $L_{\phi} = L_{u}(\vx^+, \vx^-, f_k) = a^+ (S_k) \ell(f_k(\vx^+)) + a^- (S_k) \ell(-f_k(\vx^-))$ for each $k \in [K]$, where the base loss function $\ell(t)$ is $\rho$-Lipschitz and the reweighting function $a^+ (S_k)$ (or $a^- (S_k)$) indicates its dependency on $S_k$. Then, the following inequality holds
    \begin{align*}
        \widehat{\mathfrak{R}}_{\tS}^*(L_{\phi} \circ \mathcal{F}) \leq  2 \rho \widehat{\mathfrak{R}}_{\tS, L_u}^*(\mathcal{F}) .
    \end{align*}
\end{lemma}
\begin{proof}
    Since $\widehat{\mathfrak{R}}_{\tS}^*(L_{\phi} \circ \mathcal{F}) = \frac{1}{K} \sum_{k=1}^K \widehat{\mathfrak{R}}_{\tS_k}^*(L_{\phi} \circ \mathcal{F}_k)$ and $\widehat{\mathfrak{R}}_{\tS}^*(\mathcal{F}, L_u) = \frac{1}{K} \sum_{k=1}^K \widehat{\mathfrak{R}}_{\tS, L_u}^*(\mathcal{F})$, we first prove $\widehat{\mathfrak{R}}_{\tS_k}^*(L_{\phi} \circ \mathcal{F}_k) \leq 2 \rho \widehat{\mathfrak{R}}_{\tS, L_u}^*(\mathcal{F})$ and then can get the desired result.
    
    Here we prove the inequality $\widehat{\mathfrak{R}}_{\tS_k}^*(L_{\phi} \circ \mathcal{F}_k) \leq \rho \widehat{\mathfrak{R}}_{\tS_k,L_u}^*(\mathcal{F}_k)$ following the idea in~\citet{mohri2018foundations} (Lemma~5.7, p.93), and we omit the index $k$ and the symbol $\tS_k$ or $S$ for notation simplicity in the following.
    
    First we fix a sample $(\tvx_1=(\vx_1^+, \vx_1^-),\dots,\tvx_{m}=(\vx_m^+, \vx_m^-))$, then by defintion,
    \begin{align*}
        \widehat{\mathfrak{R}}^*(L_{\phi} \circ f) & =  \eE_{\boldsymbol{\sigma}} \left[ \frac{1}{m} \sum_{j \in [J]} w_{j} \sup_{f \in \mathcal{F}} \left( \sum_{i \in I_{j}} \sigma_{i} L_{u}(\vx_i^+, \vx_i^-, f) \right) \right] \\
        & = \frac{1}{m} \sum_{j \in [J]} w_{j} \eE_{\sigma_1,\dots,\sigma_{n_j-1}} \left[ \eE_{\sigma_{n_j}} \left[ \sup_{f \in \mathcal{F}} u_{n_j - 1} (f) + \sigma_{n_j} L_{u}(\vx_{n_j}^+, \vx_{n_j}^-, f)  \right ]\right] , \qquad (\text{denote~} n_j = |I_j| \text{~for simplicity})
    \end{align*}
    where $u_{n_j- 1} (f) = \sum_{i=1}^{n_j} \sigma_i L_{u}(\vx_i^+, \vx_i^-, f)$. By the definition of the supremum, for any $\epsilon > 0$, there exists $f^1, f^2 \in \mathcal{F}$ such that
    \begin{align*}
        u_{n_j - 1} (f^1) + L_{u_2}(\vx_{n_j}^+, \vx_{n_j}^-, f^1) \leq (1 - \epsilon) \left[ \sup_{f \in \mathcal{F}} u_{n_j - 1} (f) +  L_{u}(\vx_{n_j}^+, \vx_{n_j}^-, f)  \right ] ,
    \end{align*}
    and 
    \begin{align*}
        u_{n_j - 1} (f^2) - L_{u_2}(\vx_{n_j}^+, \vx_{n_j}^-, f^2) \leq (1 - \epsilon) \left[ \sup_{f \in \mathcal{F}} u_{n_j - 1} (f) -  L_{u}(\vx_{n_j}^+, \vx_{n_j}^-, f)  \right ] .
    \end{align*}
    Thus, for any $\epsilon > 0$, by definition of $\eE_{\sigma_{n_j}}$,
    \begin{align*}
        & (1 - \epsilon) \eE_{\sigma_{n_j}} \left[ \sup_{f \in \mathcal{F}} u_{n_j - 1} (f) + \sigma_{n_j} L_{u}(\vx_{n_j}^+, \vx_{n_j}^-, f)  \right ] \\
        & = (1 - \epsilon) \left[ \frac{1}{2} \sup_{f \in \mathcal{F}} u_{n_j - 1} (f) +  L_{u}(\vx_{n_j}^+, \vx_{n_j}^-, f)  \right ] + \left[ \frac{1}{2} \sup_{f \in \mathcal{F}} u_{n_j - 1} (f) -  L_{u}(\vx_{n_j}^+, \vx_{n_j}^-, f)  \right ] \\
        & \leq \frac{1}{2} \left[ u_{n_j - 1} (f^1) +  L_{u}(\vx_{n_j}^+, \vx_{n_j}^-, f^1)) \right] + \frac{1}{2} \left[ u_{n_j - 1} (f^2) -  L_{u}(\vx_{n_j}^+, \vx_{n_j}^-, f^2) \right]. \\
        & = \frac{1}{2} \left[ u_{n_j - 1} (f^1) +  a^+ \ell(f^1(\vx_{n_j}^+)) + a^- \ell(-f^1(\vx_{n_j}^-)) \right] + \frac{1}{2} \left[ u_{n_j - 1} (f^2) -  a^+ \ell(f^2(\vx_{n_j}^+)) - a^- \ell(-f^2(\vx_{n_j}^-)) \right].
    \end{align*}
    Let $s^+ = \sgn(f^1(\vx_{n_j}^+) - f^2(\vx_{n_j}^+))$ and $s^- = \sgn(f^1(\vx_{n_j}^-) - f^2(\vx_{n_j}^-))$. Then, the previous inequality implies
    \begin{align*}
        & (1 - \epsilon) \eE_{\sigma_{n_j}} \left[ \sup_{f \in \mathcal{F}} u_{n_j - 1} (f) + \sigma_{n_j} L_{u}(\vx_{n_j}^+, \vx_{n_j}^-, f)  \right ]  \\
        & \leq \frac{1}{2} \left[ u_{n_j - 1} (f^1) + u_{n_j - 1} (f^2) + s^+ a^+ \rho (f^1(\vx_{n_j}^+) - f^2(\vx_{n_j}^+)) + s^- a^- \rho (f^1(\vx_{n_j}^-) - f^2(\vx_{n_j}^-))  \right] \qquad (\text{Lipschitz property}) \\
        & = \frac{1}{2} \left[ u_{n_j - 1} (f^1) + s^+ a^+ \rho f^1(\vx_{n_j}^+) + s^- a^- \rho f^1(\vx_{n_j}^-) \right] + \frac{1}{2} \left[ u_{n_j - 1} (f^2) - s^+ a^+ \rho f^2(\vx_{n_j}^+) - s^- a^- \rho f^2(\vx_{n_j}^-)
         \right] (\text{rearranging}) \\
        & \leq \frac{1}{2} \sup_{f \in \mathcal{F}} \left[ u_{n_j - 1} (f) + s^+ a^+ \rho f(\vx_{n_j}^+) + s^- a^- \rho f(\vx_{n_j}^-) \right] + \frac{1}{2} \sup_{f \in \mathcal{F}} \left[ u_{n_j - 1} (f) - s^+ a^+ \rho f(\vx_{n_j}^+) - s^- a^- \rho f(\vx_{n_j}^-) \right] (\text{def of sup}) \\
        & = 2 \eE_{\sigma_{n_j}^+,\sigma_{n_j}^-} \left[ u_{n_j - 1} (\tf) + \sigma_{n_j}^+ a^+ \rho f(\vx_{n_j}^+) + \sigma_{n_j}^- a^- \rho f(\vx_{n_j}^-) \right] . \qquad \qquad (\text{definition of~}\eE_{\sigma_{n_j}^+,\sigma_{n_j}^-}) \\
    \end{align*}
    Since the inequality holds for any $\epsilon > 0$, we have
    \begin{align*}
        \eE_{\sigma_{n_j}} \left[ \sup_{f \in \mathcal{F}} u_{n_j - 1} (f) + \sigma_{n_j} L_{u}(\vx_{n_j}^+, \vx_{n_j}^-, f)  \right ] \leq 2 \eE_{\sigma_{n_j}^+,\sigma_{n_j}^-} \left[ u_{n_j - 1} (\tf) + \sigma_{n_j}^+ a^+ \rho f(\vx_{n_j}^+) + \sigma_{n_j}^- a^- \rho f(\vx_{n_j}^-) \right].
    \end{align*}
    Proceeding in the same way for all other $\sigma_i$ ($i \in [I_j], i \neq n_j$) proves that 
    \begin{align*}
         \eE_{\boldsymbol{\sigma}} \left[ \frac{1}{m} \sum_{j \in [J]} w_{j} \sup_{f \in \mathcal{F}} \left( \sum_{i \in I_{j}} \sigma_{i} L_{u}(\vx_i^+, \vx_i^-, f) \right) \right] \leq  \eE_{\boldsymbol{\sigma}} \left[ \frac{1}{m} \sum_{j \in [J]} w_{j} \sup_{f \in \mathcal{F}} 2 \rho \left( \sum_{i \in I_{j}} (\sigma_{i}^+ a^+ f(\vx_{i}^+) + \sigma_{i}^- a^- f(\vx_{i}^-) )  \right) \right]. 
    \end{align*}
    By proceeding other $j \in [J]$, we can obtain the following 
    \begin{align*}
        \widehat{\mathfrak{R}}^*(L_{\phi} \circ f) & = \frac{1}{m} \sum_{j \in [J]} w_{j} \eE_{\boldsymbol{\sigma}} \left[ \sup_{f \in \mathcal{F}} \left( \sum_{i \in I_{j}} \sigma_{i} L_{u}(\vx_i^+, \vx_i^-, f) \right) \right] \\
        & \leq \frac{1}{m} \sum_{j \in [J]} w_{j} \eE_{\boldsymbol{\sigma}} \left[ \sup_{f \in \mathcal{F}} \left( \sum_{i \in I_{j}} 2 \rho (\sigma_{i}^+ a^+ f(\vx_{i}^+) + \sigma_{i}^- a^- f(\vx_{i}^-) )  \right) \right] \\
        & = 2 \rho \widehat{\mathfrak{R}}_{L_u}^*(f) .
    \end{align*}
\end{proof}

\subsubsection{Proof of Theorem~\ref{thm:learning_guarantee_u1}}

\LearningGuaranteeUone*
\begin{proof}
    Since the base loss $\ell(z)$ is bounded by $B$ and the loss $L_{u_1}(\vx^+, \vx^-, f_k) = \frac{|S_k^+|}{n} \ell \left ( f_k(\vx^+) \right ) + \frac{|S_k^-|}{n} \ell \left ( -f_k(\vx^-) \right)$, the loss $L_{\phi} = \frac{1}{\tau_S^*} L_{u_1}$ is bounded by $\frac{B}{\tau_S^*}$. Then, applying Theorem~\ref{thm:base_theorem_macroauc}, and combining Lemma~\ref{lem:app_contraction_univariate} and Proposition~\ref{pro:app_frac_rade_kernel_hypothesis_univariate_specific}, we can obtain that for any $\delta > 0$, the following generalization bound holds with probability at least $1 - \delta$ over the draw of an i.i.d. sample $S$ of size $n$:
    \begin{align*}
        \eE_S \left[ \frac{1}{\tau_S^*} \widehat{R}_{u_1} (f) \right] \leq \frac{1}{\tau_S^*} \widehat{R}_{u_1} (f) + \frac{4 \rho r \Lambda}{\tau_S^* \sqrt{n}} \left (\frac{1}{K} \sum_{k=1}^K \sqrt{\frac{1}{\tau_k}} \right )  +  \frac{3 B}{\tau_S^*} \sqrt{\frac{\log(\frac{2}{\delta})}{2n}} \left (\frac{1}{K} \sum_{k=1}^K \sqrt{\frac{1}{\tau_k}} \right ) .
    \end{align*}
    Finally, we can get the desired result by combining the above inequality and Lemma~\ref{lem:relationship_risks} (i.e., $R_{0/1} (f) \leq \eE_S \left[ \frac{1}{\tau_S^*} \widehat{R}_{u_1} (f) \right]$).
\end{proof}

\subsubsection{Proof of Theorem~\ref{thm:learning_guarantee_u2}}

\LearningGuaranteeUtwo*
\begin{proof}
    Since the base loss $\ell(z)$ is bounded by $B$ and the loss $L_{u_2}(\vx^+, \vx^-, f_k) = \ell \left ( f_k(\vx^+) \right ) + \ell \left ( -f_k(\vx^-) \right)$, the loss $L_{\phi} = L_{u_2}$ is bounded by $2 B$. Then, applying Theorem~\ref{thm:base_theorem_macroauc}, and combining Lemma~\ref{lem:app_contraction_univariate} and Proposition~\ref{pro:app_frac_rade_kernel_hypothesis_univariate_specific}, we can obtain that for any $\delta > 0$, the following generalization bound holds with probability at least $1 - \delta$ over the draw of an i.i.d. sample $S$ of size $n$:
    \begin{align*}
        R_{u_2}(f) =  \eE_S \left[ \widehat{R}_{u_2} (f) \right] \leq \widehat{R}_{u_2} (f) + \frac{8 \rho r \Lambda}{\sqrt{n}} \left (\frac{1}{K} \sum_{k=1}^K \sqrt{\frac{1}{\tau_k}} \right ) + 
        6 B \sqrt{\frac{\log(\frac{2}{\delta})}{2n}} \left ( \sqrt{\frac{1}{K} \sum_{k=1}^K \frac{1}{\tau_k}} \right )
    \end{align*}
    Finally, we can get the desired result by combining the above inequality and Lemma~\ref{lem:relationship_risks} (i.e., $R_{0/1} (f) \leq R_{u_2}(f)$).
\end{proof}

\subsubsection{Proof of Corollary~\ref{cor:learning_guarantee_u_balanced}}

\LearningGuaranteeUBalanced*
\begin{proof}
    It is straightforward to get the result by applying the Theorem~\ref{thm:learning_guarantee_u1} and Theorem~\ref{thm:learning_guarantee_u2} by plugging $\tau_k = \frac{1}{2}$.
\end{proof}

\subsubsection{Proof of Corollary~\ref{cor:learning_guarantee_u1_extreme_imbalanced}}

\LearningGuaranteeUOneExtremeImbalanced*
\begin{proof}
    It is straightforward to get the result by applying the Theorem~\ref{thm:learning_guarantee_u1} by plugging $\tau_k = \frac{1}{n}$.
\end{proof}

\subsubsection{Proof of Corollary~\ref{cor:learning_guarantee_u2_extreme_imbalanced}}
\label{sec:app_learning_guarantee_u2_extreme_imbalanced}

\begin{restatable}[\textbf{Learning guarantee of $\mathcal{A}^{u_2}$ in extremely imbalanced case}]
    {corollary}{LearningGuaranteeUTwoExtremeImbalanced}
\label{cor:learning_guarantee_u2_extreme_imbalanced}
    Assume the loss $L_{\phi} = L_{u_2}$, where $L_{u_2}$ is defined in Eq.\eqref{eq:surrogate_u2}. Besides, {\rm Assumption~\ref{assump_common}} holds and suppose $S$ is extremely imbalanced. Then, for any $\delta > 0$, with probability at least $1 - \delta$ over the draw of an i.i.d. sample $S$ of size $n$, the following generalization bound holds for any $f \in \mathcal{F}$: 
    \begin{align*}
        R_{0/1} (f) \leq R_{u_2} (f) \leq & \widehat{R}_{u_2} (f) + 8 \rho r \Lambda + 6 B \sqrt{\frac{\log(\frac{2}{\delta})}{2}} .
    \end{align*}
\end{restatable}
\begin{proof}
    It is straightforward to get the result by applying the Theorem~\ref{thm:learning_guarantee_u2} and Theorem~\ref{thm:learning_guarantee_u2} by plugging $\tau_k = \frac{1}{n}$.
\end{proof}

\section{Additional Experimental Results}
\label{sec:app_experiments}

\subsection{Label-wise class imbalance illustrations of benchmark datasets}
\label{sec:app_label_wise_class_imbalance_illustration}


The label-wise class imbalance levels of all the benchmark datasets are illustrated in Figure~\ref{fig:imbalance_benchmarks_full}.

\begin{figure}[h]
        \centering
        \begin{subfigure}[CAL500]
            {\includegraphics[scale=0.25]{images/cal500.eps}}
        \end{subfigure}
        \begin{subfigure}[emotions]
            {\includegraphics[scale=0.25]{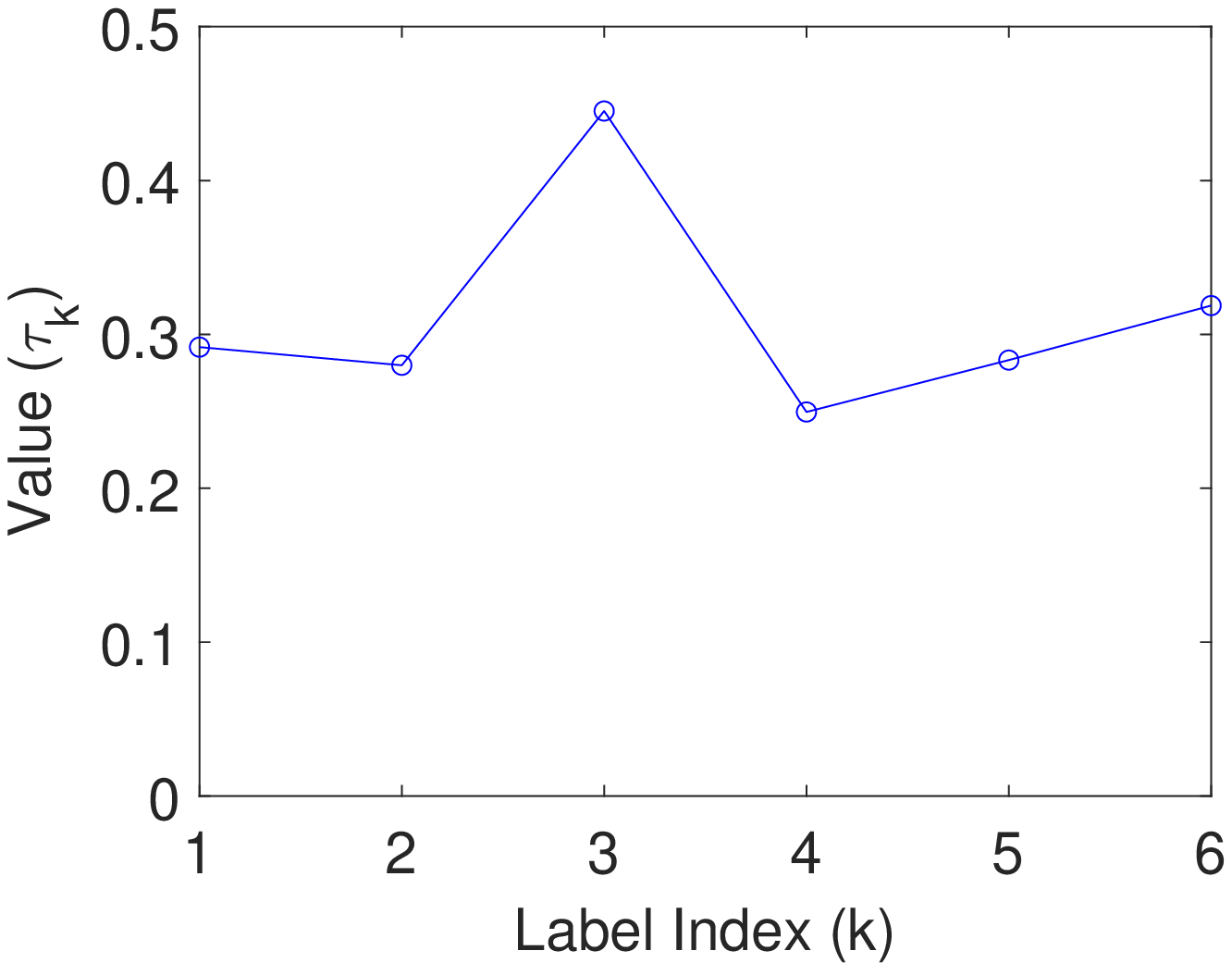}}
        \end{subfigure}
        \begin{subfigure}[image]
            {\includegraphics[scale=0.25]{images/image.eps}}
        \end{subfigure}
        \begin{subfigure}[scene]
            {\includegraphics[scale=0.25]{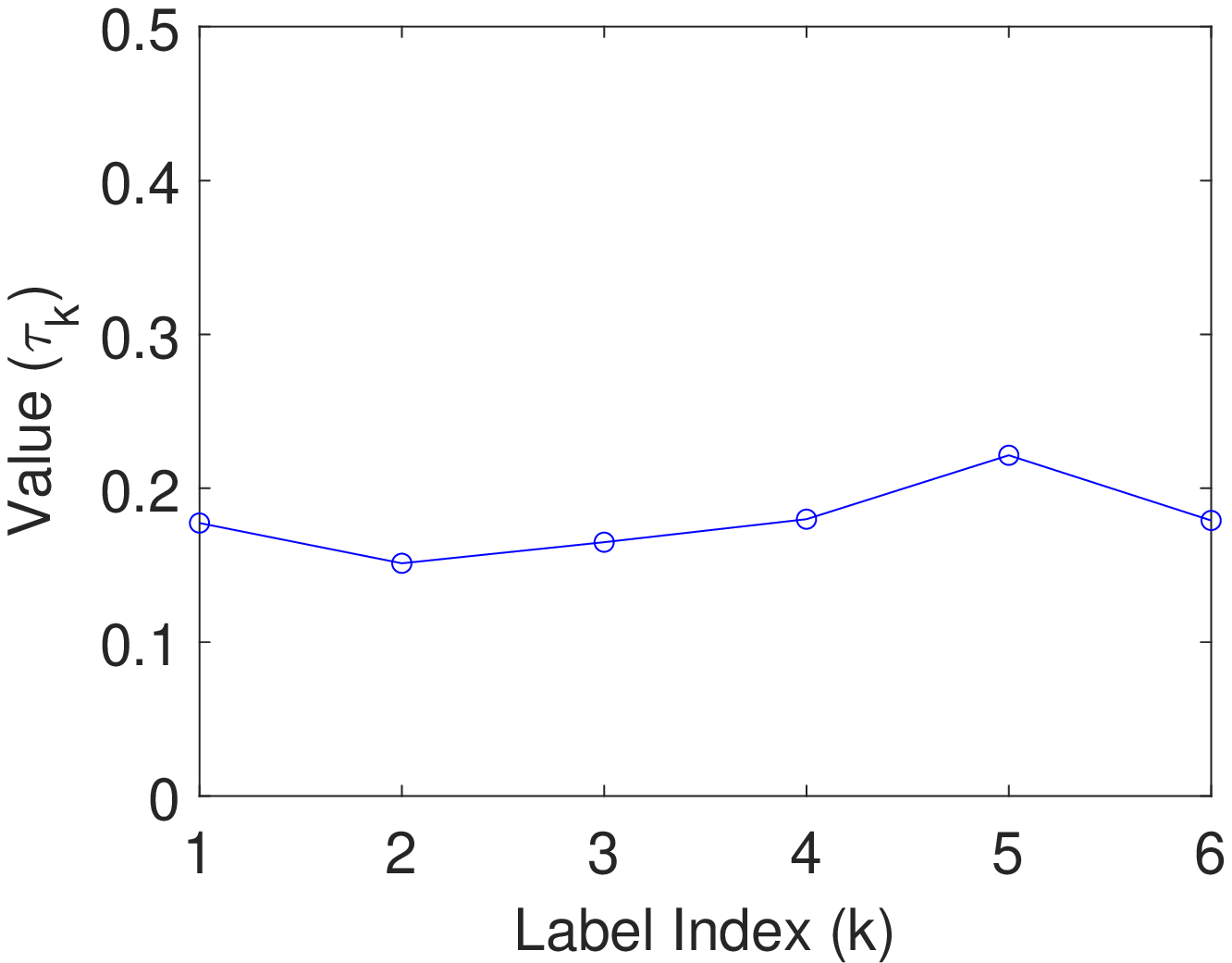}}
        \end{subfigure}
        \begin{subfigure}[yeast]
            {\includegraphics[scale=0.25]{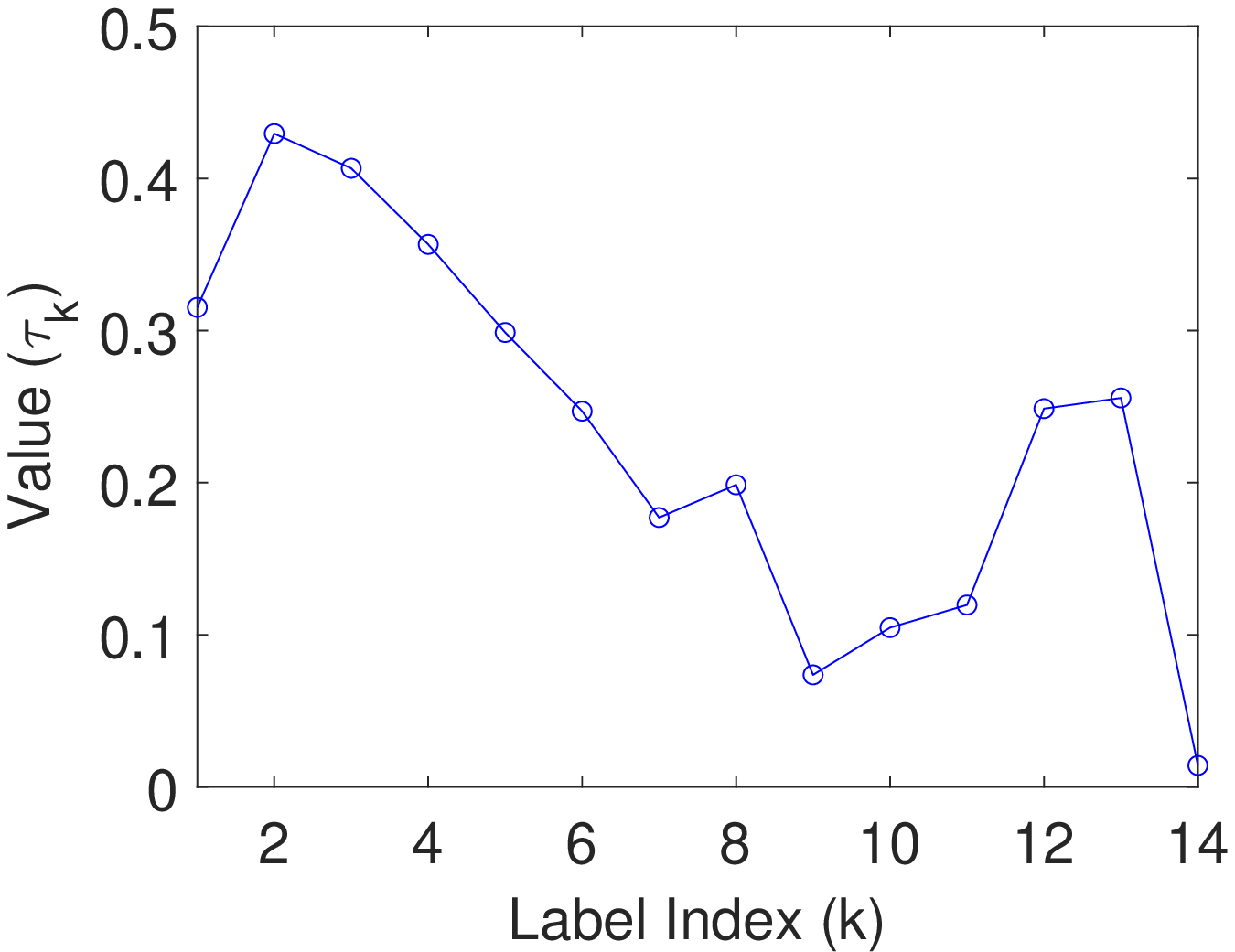}}
        \end{subfigure}
        \begin{subfigure}[enron]
            {\includegraphics[scale=0.25]{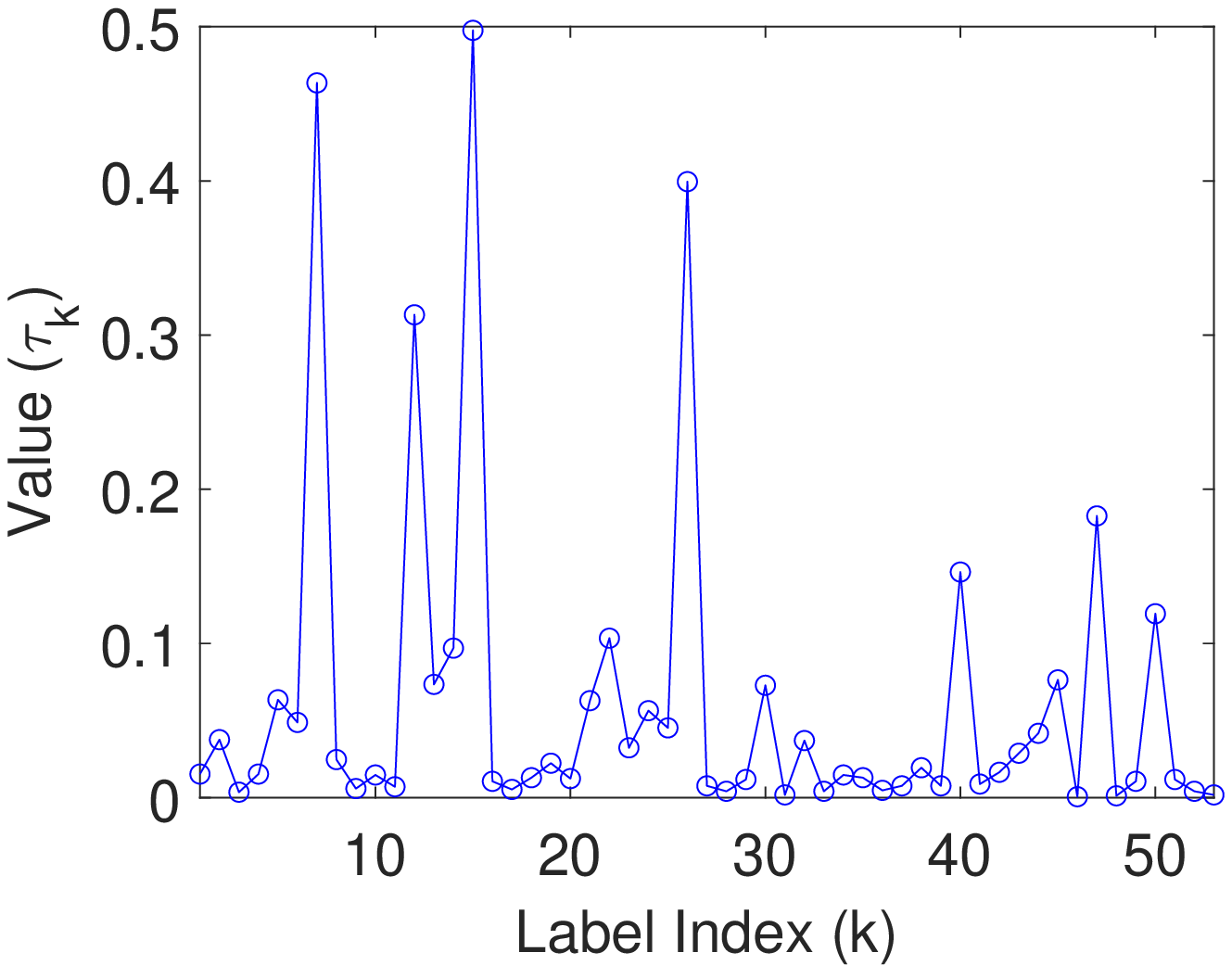}}
        \end{subfigure}
        \begin{subfigure}[rcv1-s1]
            {\includegraphics[scale=0.25]{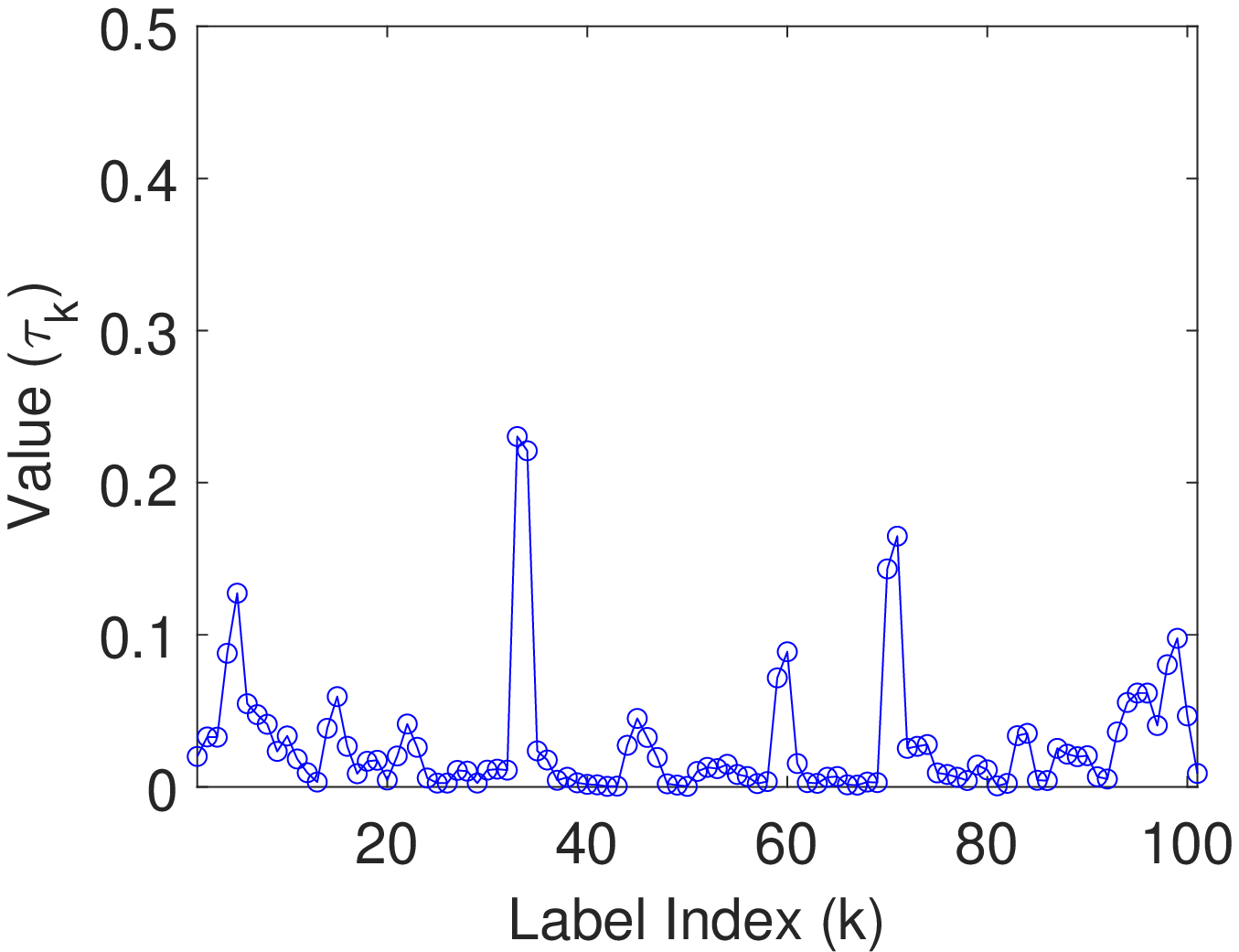}}
        \end{subfigure}
        \begin{subfigure}[bibtex]
            {\includegraphics[scale=0.25]{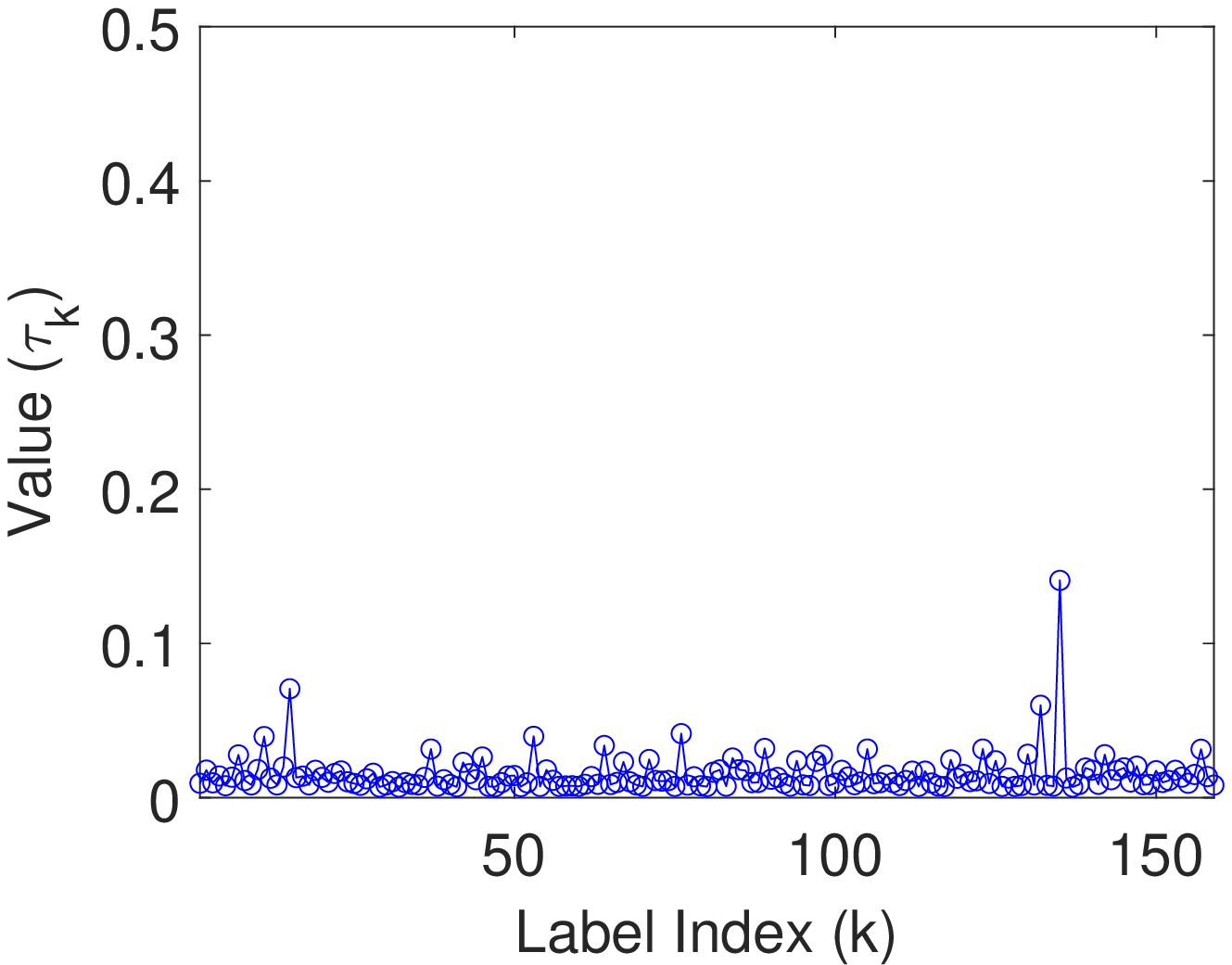}}
        \end{subfigure}
        \begin{subfigure}[corel5k]
            {\includegraphics[scale=0.25]{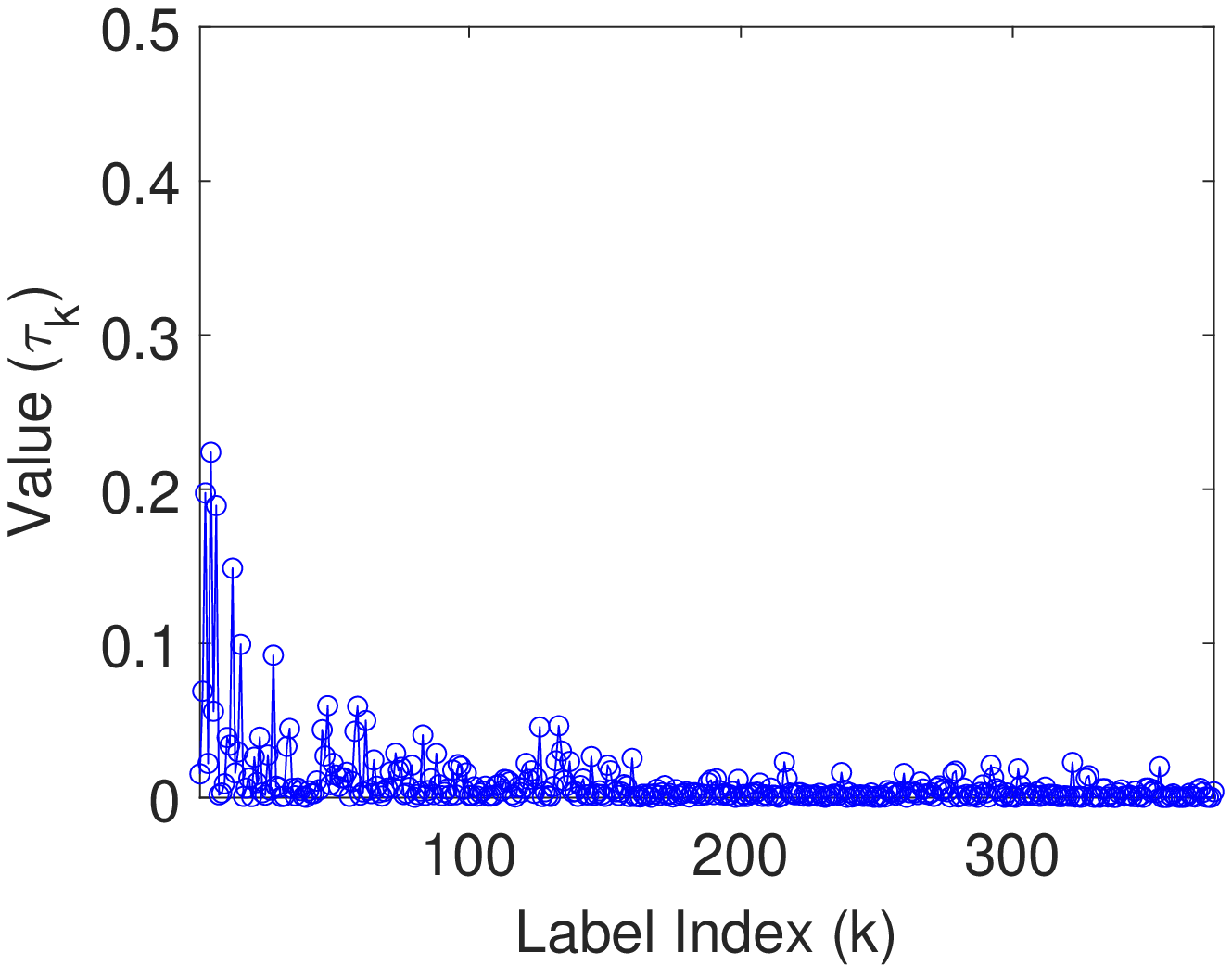}}
        \end{subfigure}
        \begin{subfigure}[delicious]
            {\includegraphics[scale=0.25]{images/delicious.eps}}
        \end{subfigure}
        \caption{The illustration of the label-wise class imbalance of each benchmark dataset.}
        \label{fig:imbalance_benchmarks_full}
\end{figure}

\subsection{Experimental results about the absolute value of bounds}
\label{sec:app_absolute_value_bounds}

Here we report the mean upper bound values of three algorithms for the benchmark datasets in Table~\ref{tab:empirical_bounds}. From the experimental results, we can observe that the absolute values might not reflect the true generalization risk (i.e., bigger than $1$), but they can still offer valuable insights into these algorithms by comparing the order of dependent factors under the same framework.
\begin{table*}[h!]
\scriptsize
\caption{The mean upper bound values of three algorithms for the benchmark datasets. We set $\delta = 0.01$.}
\label{tab:empirical_bounds}
\vskip 0.15in
\begin{center}
\begin{small}
\begin{tabular}{lccc}
\toprule
Dataset & $\mathcal{A}^{pa}$ &  $\mathcal{A}^{u_1}$ & $\mathcal{A}^{u_2}$ \\
\midrule
    CAL500 & $13.0$ & $7082.4$ & $36.3$ \\
    emotions & $13.6$ & $56.8$ & $33.7$ \\
    image & $13.1$ & $191.0$ & $30.7$ \\
    scene & $7.4$ & $146.2$ & $19.6$ \\
    yeast & $3.6$ & $375.2$ & $12.1$ \\
    enron & - & $37016.0$ & $69.0$ \\
    rcv1-s1 & - & $96284.0$ & $20.2$  \\
    bibtex & - & $207.4$ & $55.2$  \\
    corel5k & - & $97359.0$ & $3.1$  \\
    delicious & - & $293000.0$ & $82.8$ \\
\bottomrule
\end{tabular}
\end{small}
\end{center}
\vskip -0.1in
\end{table*}

\section{Additional Related Work}
\label{sec:app_addtional_related_work}

Recently, there have been some works on the McDiarmid-type concentration inequality for data with graph dependence in quite general settings, e.g., ~\citet{zhang2019mcdiarmid}.
However, to the best of our knowledge, it is not clear to apply the conclusion of the existing work~\cite{zhang2019mcdiarmid} (i.e., Theorem 3.6 for the general dependency graph on Page 5) to get the explicit generalization bound in the Macro-AUC maximization of MLC, as the forest complexity of the dependency graph might be non-trivial to estimate in this case. In contrast, our proposed new McDiarmid-type concentration inequality (i.e., Theorem~\ref{thm:new_mcdiarmid}) is easy to apply in this case.

Note that our proposed concentration inequality (i.e., Theorem~\ref{thm:new_mcdiarmid}) is not a corollary of the existing one (i.e., Theorem 3.6 in~\cite{zhang2019mcdiarmid}). They are complementary to each other with different assumptions. Although~\citet{zhang2019mcdiarmid} is general for the general dependency graph, ours consider the particular case with additional assumptions by constraining the function 
 and the dependency graph (i.e., the assumptions (1) and (2) in Theorem~\ref{thm:new_mcdiarmid}), which cannot be induced by~\citet{zhang2019mcdiarmid}.


\end{document}

%% file: math_command.tex
\def\va{{\mathbf{a}}}

\def\sR{\mathbb{R}}

\def\va{{\mathbf{a}}}

\def\vw{{\mathbf{w}}}
\def\vx{{\mathbf{x}}}

\def\mA{{\mathbf{A}}}

\def\mW{{\mathbf{W}}}
\def\mX{{\mathbf{X}}}

\def\eE{\mathop{\mathbb{E}}\limits}

\def\pP{\mathbb{P}}

\def\tf{\tilde{f}}
\def\tS{\tilde{S}}
\def\wtF{\widetilde{\mathcal{F}}}

\def\tvx{\tilde{\mathbf{x}}}
\def\ty{\tilde{y}}